\documentclass[11pt]{article}

\usepackage{amsmath,amsbsy,amsfonts,amssymb,amsthm,dsfont,fullpage,units}
\usepackage{authblk}
\usepackage[round]{natbib}
\usepackage{algorithm,algorithmic}
\usepackage{color}
\usepackage{lscape}
\def\E{\mathbb{E}}
\def\R{\mathbb{R}}

\def\I{\mathrm{I}}

\def\F{{\operatorname{F}}}

\def\sphere{\mathcal{S}}
\DeclareMathOperator{\range}{range}
\def\eps{\epsilon}

\newcommand\dotp[1]{\langle #1 \rangle}
\def\t{{\scriptscriptstyle\top}}
\def\h{\hat}
\def\tl{\tilde}
\def\wh{\widehat}
\def\wt{\widetilde}
\def\pmin{p_{\min}}
\def\canO{{\tl O}}

\DeclareMathOperator{\diag}{diag}

\DeclareMathOperator{\Pairs}{Pairs}
\DeclareMathOperator{\Triples}{Triples}
\DeclareMathOperator{\Pairsa}{\Pairs_{\alpha_0}}
\DeclareMathOperator{\Triplesa}{\Triples_{\alpha_0}}
\DeclareMathOperator{\Quad}{Quadruples}

\newcommand\hidden{h}

\newtheorem{assumption}{Assumption}[section]
\newtheorem{lemma}{Lemma}[section]

\newtheorem{theorem}{Theorem}[section]
\theoremstyle{definition}

\newtheorem{definition}{Definition}

\theoremstyle{remark}
\newtheorem{remark}{Remark}

\title{A Spectral Algorithm for Latent Dirichlet Allocation\footnote{Previous title: ``Two SVDs Suffice:
Spectral decompositions for probabilistic
topic modeling and latent Dirichlet allocation''.}}
\date{}
\author[1]{Animashree Anandkumar}
\author[3]{Dean P. Foster}
\author[2]{Daniel Hsu}
\author[2]{\\Sham M.~Kakade}
\author[4]{Yi-Kai Liu}
\affil[1]{Department of EECS, University of California, Irvine}
\affil[2]{Microsoft Research, New England}
\affil[3]{Department of Statistics, Wharton School, University of
  Pennsylvania}
\affil[4]{National Institute of Standards and Technology, Gaithersburg, MD~\footnote{Contributions to this work by NIST, an agency of the US
  government, are not subject to copyright laws. }}

\begin{document}
\maketitle

\begin{abstract}
The problem of topic modeling can be seen as a generalization of the
clustering problem, in that it posits that observations are generated due
to multiple latent factors (\emph{e.g.}, the words in each document are
generated as a mixture of \emph{several} active topics, as opposed to just
one). This increased representational power comes at the cost of a more
challenging unsupervised learning problem of estimating the topic
probability vectors (the distributions over words for each topic), when
only the words are observed and the corresponding topics are hidden.

We provide a simple and efficient learning procedure that is guaranteed to
recover the parameters for a wide class of mixture models, including the
popular latent Dirichlet allocation (LDA) model. For LDA, the procedure
correctly recovers both the topic probability vectors and the prior over
the topics, using only trigram statistics (\emph{i.e.}, third order moments, which
may be estimated with documents containing just three words). The method,
termed Excess Correlation Analysis (ECA), is based on a spectral
decomposition of low order moments (third and fourth order) via two
singular value decompositions (SVDs). Moreover, the algorithm is scalable
since the SVD operations are carried out on $k\times k$ matrices, where $k$
is the number of latent factors (e.g. the number of topics), rather than in
the $d$-dimensional observed space (typically $d \gg k$).
\end{abstract}


\section{Introduction}

There is general agreement that there are multiple unobserved or latent
factors affecting observed data.
Mixture models offer a powerful framework to incorporate
the effects of these latent variables. A family of mixture models, popularly known as {\em
  topic models}, has generated broad interest on both theoretical and
practical fronts.

Topic models incorporate latent variables, the topics, to
explain the observed co-occurrences of words in documents. They posit
that each document has a mixture of active topics (possibly sparse)
and that each active topic determines the occurrence of words in the
document. Usually, a Dirichlet prior is assigned to the distribution
of topics in documents, giving rise to the so-called latent Dirichlet
allocation (LDA)~\citep{blei2003latent}. These models possess a rich
representational power since they allow for the words in each document
to be generated from more than one topic (\emph{i.e.}, the model permits
documents to be about multiple topics).  This increased
representational power comes at the cost of a more challenging
unsupervised estimation problem, when only the words are observed and
the corresponding topics are hidden.

In practice, the most common estimation procedures are based on
finding maximum likelihood (ML) estimates, through either local search
or sampling based methods, \emph{e.g.}, Expectation-Maximization
(EM)~\citep{RW84}, Gibbs sampling~\citep{asuncion2011distributed}, and
variational approaches~\citep{hoffman2010online}.  Another body of
tools is based on matrix
factorization~\citep{hofmann1999plsa,lee99}. For document modeling,
typically, the goal is to form a sparse decomposition of a term by
document matrix (which represents the word counts in each document)
into two parts: one which specifies  the active topics  in
each document and the other which specifies the
distributions of words under each topic.

This work provides an alternative approach to parameter recovery based
on the method of moments~\citep{Lindsay89,LB93}, which attempts to
match the observed moments with those posited by the model. Our
approach does this efficiently through a spectral decomposition of the
observed moments through two singular value decompositions. This
method is simple and efficient to implement, based on only low order
moments (third or fourth order), and is guaranteed to recover the
parameters of a wide class of mixture models, including the LDA
model. We exploit exchangeability of the observed variables and, more
generally, the availability of multiple views drawn independently from
the same hidden component.

\subsection{Summary of Contributions}

We present an approach known as Excess Correlation Analysis (ECA)
based on the knowledge of low order moments between the
observed variables, assumed to be exchangeable (or, more generally,
drawn from a multi-view mixture model). ECA differs from Principal
Component Analysis (PCA) and Canonical Correlation Analysis (CCA) in
that it is based on two singular value decompositions: the first SVD
whitens the data (based on the correlation between two variables) and
the second SVD utilizes higher order moments (based on third or fourth
order) to find directions which exhibit moments that are in
\emph{excess} of those suggested by a Gaussian distribution.  Both
SVDs are performed on matrices of size $k\times k$, where $k$ is the
number of latent factors, making the algorithm scalable (typically the
dimension of the observed space $d \gg k$).

The method is applicable to a wide class of mixture models including
exchangeable and multi-view models. We first consider the class of
exchangeable variables with independent latent factors, such as a
latent Poisson mixture model (a natural Poisson model for generating
the sentences in a document, analogous to LDA's multinomial model for
generating the words in a document). We establish that a spectral
decomposition, based on third or fourth order central moments,
recovers the parameters for this model class. We then consider latent
Dirichlet allocation and show that a spectral decomposition of a
modified third order moment (exactly) recovers both the probability
distributions over words for each topic and the Dirichlet prior.  Note
that to obtain third order moments, it suffices for documents to
contain just $3$ words. Finally, we present extensions to multi-view
models, where multiple views drawn independently from the same latent
factor are available. This includes the case of both pure topic models
(where only one active topic is present in each document) and discrete
hidden Markov models. For this setting, we establish that ECA
correctly recovers the parameters and is simpler than the eigenvector
decomposition methods of~\cite{AHKparams}.

Finally, ``plug-in'' moment estimates can be used with sampled
data. Section~\ref{sec:samp_comp} provides a sample complexity of the method showing that
estimating the third order moments is not as difficult as it might
naively seem since we only need a $k\times k$ matrix to be accurate.

Some preliminary experiments that illustrate the efficacy of the proposed
algorithm are given in the appendix.


\subsection{Related Work}

For the case of a single topic per document, the work of
\cite{Papadimitriou98latentsemantic} provides the first guarantees of
recovering the topic distributions (\emph{i.e.}, the distributions over words
corresponding to each topic), albeit with a rather stringent
separation condition (where the words in each topic are essentially
non overlapping). Understanding what separation conditions (or lack thereof)
permit efficient learning is a natural question; in the
clustering literature, a line of work has focussed on understanding
the relation between the separation of the mixture components and the
complexity of learning. For clustering, the first 
learnability result~\citep{Das99} was under a somewhat strong separation condition; 
a subsequent line of results
relaxed~\citep{AK01,DS07,VW02,KSV05,AM05,CR08,BV08,CKLS09} or removed
these conditions~\citep{KMV10,BS10,MV10}; roughly
speaking, the less stringent the separation condition assumed, the more
difficult the learning problem is, both computationally and
statistically. For the topic modeling problem in which only a single
topic is present per document,
 \cite{AHKparams} provides an algorithm for learning topics with no
separation (only a certain full rank assumption is utilized).

For the case of latent Dirichlet allocation (where multiple topics are
present in each document), the recent work of \cite{Arora:1439928}
provides the first provable result under a certain natural separation
condition. The notion of separation utilized is based on the existence
of ``anchor words'' for topics --- essentially , each topic contains
words that appear (with reasonable probability) only in that topic
(this is a milder assumption than that in
 \cite{Papadimitriou98latentsemantic}). Under this assumption,
\cite{Arora:1439928} provide the first provably correct algorithm for
learning the topic distributions. Their work also justifies the use of
non-negative matrix (NMF) as a procedure for this problem
(the original motivation for NMF was as a topic modeling algorithm,
though, prior to this work, formal guarantees as such were rather
limited). Furthermore, \cite{Arora:1439928} provides results for
certain correlated topic models. 

Our approach makes further progress on this problem by providing an
algorithm which requires no separation condition. The underlying
approach we take is a certain diagonalization technique of the
observed moments. We know of at least three different settings which
utilize this idea for parameter estimation.

\cite{Chang96} utilizes eigenvector methods for discrete Markov models
of evolution, where the models involve multinomial distributions. The
idea has been extended to other discrete mixture models such as
discrete hidden Markov models (HMMs) and mixture models with single active
topics (see \citet{MR06,HKZ09,AHKparams}).  A key idea in
\cite{Chang96} is the ability to handle multinomial distributions,
which comes at the cost of being able to handle only certain single
latent factor/topic models (where the latent factor is in only one of $k$
states, such as in HMMs). For these single topic models, the work in
\cite{AHKparams} shows how this method is quite general in that the
noise model is essentially irrelevant, making it
applicable to both discrete models like HMMs and certain Gaussian
mixture models.

The second setting is the body of algebraic methods used for the
problem of blind source
separation~\citep{Cardoso96independentcomponent}. These approaches
rely on tensor decomposition approaches (see
 \cite{Comon:book}) tailored to independent source separation
  with additive noise (usually Gaussian). Much of literature
focuses on   understanding the effects of measurement noise (without assuming knowledge of their statistics) on
the tensor decomposition, which often requires more sophisticated
algebraic tools.

\cite{FriezeLinear} also utilize these ideas for learning the columns
of a linear transformation (in a noiseless setting). This work
provides a different efficient algorithm, based on a certain
ascent algorithm (rather than joint diagonalization approach, as in
 \citep{Cardoso96independentcomponent}).

The underlying insight that our method exploits is that we have
exchangeable (or multi-view) variables, \emph{e.g.}, we have multiple words
(or sentences) in a document, which are drawn independently from the
same hidden state. This  allows us to borrow from both the
ideas in \cite{Chang96} and in \cite{Cardoso96independentcomponent}.
In particular, we show that the ``topic'' modeling problem
exhibits a rather simple algebraic solution, where only two SVDs
suffice for parameter estimation. Moreover, this approach
also simplifies the algorithms in \citet{MR06,HKZ09,AHKparams},
in that the eigenvector methods are no longer necessary (\emph{e.g.}, the
approach leads to methods for parameter estimation in HMMs with only
two SVDs rather than using eigenvector approaches, as in previous
work).

Furthermore, the exchangeability assumption permits us to have
\emph{arbitrary} noise models (rather than additive Gaussian noise,
which are not appropriate for multinomial and other discrete
distributions). A key technical contribution is that we show how the
basic diagonalization approach can be adapted to Dirichlet
models, through a rather careful construction. This construction
bridges the gap between the single topic models (as in
 \cite{Chang96,AHKparams}) and the independent factor model.


More generally, the multi-view approach has been exploited in previous
works for semi-supervised learning and for learning mixtures of
well-separated distributions (\emph{e.g.}, as in \cite{ando07,
  kakadecca,CR08,CKLS09}). These previous works essentially use
variants of canonical correlation analysis~\citep{Hotelling35}
between two views. This work shows that having a third view of the
data permits rather simple estimation procedures with guaranteed
parameter recovery.

\section{The Exchangeable and Multi-view Models} \label{sec:models}

We have a random vector $h = (h_1,h_2,\dotsc,h_k)^\t \in \R^k$.  This
vector specifies the latent factors (\emph{i.e.}, the hidden state), where $h_i$ specifies
the value taken by $i$-th factor. Denote the variance of $h_i$ as
\[
\sigma_i^2 = \E[(h_i-\E[h_i])^2]
\]
which we assume to be strictly positive, for each $i$, 
and denote the higher $l$-th central moments of $h_i$ as:
\begin{eqnarray*}
\mu_{i,l} & := &  \E[(h_i-\E[h_i])^l]
\end{eqnarray*}
At most, we only use the first four moments in our analysis.

Suppose we also have a sequence of  \emph{exchangeable} random vectors $\{x_1, x_2,
x_3, x_4, \dotsc \} \in \R^d$; these are considered to be the observed
variables. Assume throughout that $d \geq k$; that $x_1, x_2, x_3,
x_4, \dotsc \in \R^d$ are conditionally independent given $h$; and there
exists a matrix $O\in \R^{d\times k}$ such that
\[
\E[x_v|h] = Oh
\]
for each $v \in \{1,2,3,4,\dotsc\}$.  Throughout, we make the following
assumption.
\begin{assumption}
$O$ is full rank.
\end{assumption}
\noindent
This is a mild assumption, which allows for identifiability of the
columns of $O$. The goal is to estimate the matrix $O$,
sometimes referred to as the topic
matrix.

Importantly, we make no assumptions on the noise model. In particular,
we do not assume that the noise is additive (or that the noise is
independent of $h$).


\subsection{Independent Latent Factors}

Here, suppose that $h$ has a product distribution, \emph{i.e.}, each component
of $h_i$ is independent from the rest. Two important examples of this
setting are as follows:

 {\bf (Multiple) mixtures of Gaussians:}  Suppose $x_v = Oh+\eta$, where
$\eta$ is Gaussian noise and $h$ is a binary vector (under a product distribution).  Here, the $i$-th
column $O_i$ can be considered to be the mean of the $i$-th Gaussian
component.  This is somewhat different model than the classic mixture of $k$-Gaussians, as
the model now permits any number of Gaussians to be responsible for generating the
hidden state (\emph{i.e.}, $h$ is permitted to be any of the $2^k$ vectors on
the hypercube, while in the classic mixture problem, only one
component is responsible. However, this model imposes the independent
factor constraint.).  We may also allow $\eta$ to be
heteroskedastic (\emph{i.e.}, the noise may depend on $h$, provided the
linearity assumption $\E[x_v|h] = Oh$ holds.)


{\bf (Multiple) mixtures of Poissons:} Suppose $[Oh]_j$ specifies the
Poisson rate of counts for $[x_v]_j$. For example, $x_v$ could be a
vector of word counts in the $v$-th sentence of a document (where
$x_1,x_2, \ldots$ are words counts of a sequence sentences). Here, $O$ would
be a matrix with positive entries, and $h_i$ would scale the rate at
which topic $i$ generates words in a sentence (as specified by the
$i$-th column of $O$). The linearity assumption is satisfied as
$\E[x_v|h] = Oh$ (note the noise is not additive in this
case). Here, multiple topics may be responsible for generating the
words in each sentence. This model provides a natural variant of
LDA, where the distribution over $h$ is
a product distribution (while in LDA, $h$ is a probability vector).

\subsection{The Dirichlet  Model}

Now suppose the hidden state $h$ is a distribution itself, with a
density specified by the Dirichlet distribution with parameter $\alpha
\in \R_+^k$ ($\alpha$ is a strictly positive real vector). We
often think of $h$ as a distribution over topics. Precisely, the
density of $\hidden \in \Delta^{k-1}$ (where the probability simplex
$\Delta^{k-1}$ denotes the set of possible distributions over $k$
outcomes) is specified by:
\[
p_\alpha(\hidden) := \frac1{Z(\alpha)} \prod_{i=1}^k
\hidden_i^{\alpha_i-1}
\]
where
\[
Z(\alpha)
:= \frac{\prod_{i=1}^k \Gamma(\alpha_i)}{\Gamma(\alpha_0)}
\]
and
\[
\alpha_0 := \alpha_1 + \alpha_2 + \dotsb + \alpha_k \, . 
\]
Intuitively, $\alpha_0$ (the sum of the ``pseudo-counts'') is a crude
measure of the uniformity of the distribution.  As
$\alpha_0\rightarrow 0$, the distribution degenerates to one over pure
topics (\emph{i.e.}, the limiting density is one in which, with probability
$1$, precisely one coordinate of $h$ is $1$ and the rest are $0$).

{\bf Latent Dirichlet Allocation:} LDA makes the further
assumption that each random variable
$x_1,x_2, x_3, \dotsc$ takes on discrete values out of $d$ outcomes
(\emph{e.g.}, $x_v$ represents what the $v$-th word in a document is, so
$d$ represents the number of words in the language). Each
column of $O$ represents a distribution over the outcomes (\emph{e.g.}, these
are the topic probabilities).  The sampling procedure is specified as
follows: First, $h$ is sampled according to the Dirichlet
distribution. Then, for each $v$, independently sample $i \in
\{1,2,\dotsc k\}$ according to $h$, and, finally, sample $x_v$ according to
the $i$-th column of $O$. Observe this model falls into our setting:
represent $x_v$ with a ``hot'' encoding where $[x_v]_j=1$ if and only
if the $v$-th outcome is the $j$-th word in the vocabulary.  Hence,
$\Pr([x_v]_j=1 | h) = [Oh]_j$ and $\E[x_v|\hidden] = O\hidden
$. (Again, the noise model is not additive).

\subsection{The Multi-View Model}

The multi-view setting can be considered an extension of the
exchangeable model. Here, the random vectors $\{x_1, x_2, x_3, \dotsc
\} $ are of dimensions $d_1,d_2,d_3, \dotsc$.  Instead of a single $O$
matrix, suppose for each $v \in \{1,2,3,\dotsc\}$ there exists an
$O_v\in \R^{d_v\times k}$ such that
\[
\E[x_v|h] = O_v h
\]
Throughout, we make the following assumption.
\begin{assumption}
$O_v$ is full rank for each $v$.
\end{assumption}
Even though the variables are no longer exchangeable, the setting
shares much of the statistical structure as the exchangeable one;
furthermore, it allows for significantly richer models.  For example,
\cite{AHKparams} consider a special case of this multi-view model
(where there is only one topic present in $h$) for the purposes of
learning hidden Markov models.

{\bf A simple factorial HMM:} Here, suppose we have a time series of
random hidden vectors $h_1,h_2,h_3,\ldots$ and observations
$x_1,x_2,x_3,\ldots$ (we slightly abuse notation as $h_1$ is a
vector). Assume that each factor $[h_t]_i \in \{-1,1\}$. The model
parameters and evolution are specified as follows: We have an initial
(product) distribution over the first $h_1$. The ``factorial''
assumption we make is that each factor $[h_t]_i$ evolves
independently; in particular, for each component $i$, there are 
(time independent) transition probabilities $p_{i,1\rightarrow -1}$
and $p_{i,1\rightarrow -1}$. Also suppose that $\E[x_t |h_t] = O h_t$
(where, again, $O$ does not depend on the time).

To learn this model, consider the first three observations
$x_1,x_2,x_3$.  We can embed this three timestep model into the
multiview model using a single hidden state, namely $h_2$, and, with
an appropriate construction (of $O_1,O_2,O_3$ and means
shifts of $x_v$ to make the linearity assumption hold). Furthermore, if we recover $O_1, O_2, O_3$ we can recover $O$
and the transition model. See \cite{AHKparams} for further discussion
of this idea (for the single topic case).

\section{Identifiability}

The underlying question here is: what may we hope to recover about $O$
with only knowledge of the distribution on $x_1, x_2, x_3, \dotsc$. At
best, we could only recover the columns of $O$ up to permutation.  At
the other extreme, suppose no a priori knowledge of the distribution
of $h$ is assumed (\emph{e.g.}, it may not even be a product
distribution). Here, at best, we can only recover the range of $O$. In
particular, suppose $h$ is distributed according to a multivariate
Gaussian, then clearly the columns of $O$ are not identifiable. To see
this, transform $O$ to $OM$ (where $M$ is any $k\times k$ invertible
matrix) and transform the distribution on $h$ (by $M^{-1}$); after
this transformation, the distribution over $x_v$ is unaltered and the
distribution on $h$ is still a multivariate Gaussian. Hence, $O$ and
$OM$ are indistinguishable from any observable statistics.  (These
issues are well understood in setting of independent source
separation, for additive noise models without exchangeable
variables. See \cite{Comon:book}).

Thus, for the columns of $O$ to be identifiable, the distribution on
$h$ must have some non-Gaussian statistical properties.  We consider
three cases. 
In the independent
factor model, we consider the cases when $h$ is skewed and when $h$
has excess kurtosis. We also consider the case that $h$ is
Dirichlet distributed.


\section{Excess Correlation Analysis (ECA)}

We now present exact and efficient algorithms for recovering $O$.  The
algorithm is based on two singular value decompositions: the first SVD
whitens the data (based on the correlation between two variables) and
the second SVD is carried out on higher order moments (based on third
or fourth order).  We start with the case of independent factors, as
these algorithms make the basic diagonalization approach clear.

As discussed in the Introduction, these approaches can been seen as
extensions of the methodologies in
\cite{Chang96,Cardoso96independentcomponent}.  Furthermore, as we
shall see, the Dirichlet distribution bridges between the single
topic models (as in \cite{Chang96,AHKparams}) and the independent
factor model.

Throughout, we use $A^+$ to denote the pseudo-inverse:
\begin{equation}\label{eq:pseudo}
A^+ = (A^\t A)^{-1}A^\t
\end{equation}
for a matrix $A$ with linearly independent columns (this allows us
to appropriately invert non-square matrices).

\subsection{Independent and Skewed Latent Factors}

\begin{algorithm} [t]
\caption{ECA, with skewed factors} \label{alg:skew}
\begin{algorithmic}
\STATE {\bf Input:}
vector $\theta \in \R^k$; the moments $\Pairs$ and $\Triples(\eta)$

\begin{enumerate}
\STATE {\bf Dimensionality Reduction:}
Find a matrix $U\in \R^{d \times k}$ such that
\[
\textrm{Range}(U)=\textrm{Range}(\Pairs).
\]
(See Remark~\ref{remark:range} for a fast procedure.)
\STATE {\bf Whiten:} Find $V\in \R^{k \times k}$ so
$V^\t (U^\t \Pairs U) V $
is the $k\times k$ identity matrix. Set:
\[
W =  UV
\]
\STATE {\bf SVD:}  Let $\Lambda$ be the set of
(left) singular vectors, with \emph{unique} singular
values, of
\[
W^\t \Triples(W \theta) W
\]
\STATE {\bf Reconstruct:} Return the set $\wh O$:
\[
\wh O = \{ \ (W^+)^\t \lambda \ : \lambda \in \Lambda\}
\]
where $W^+$ is the pseudo-inverse (see Eq~\ref{eq:pseudo}).
\end{enumerate}
\end{algorithmic}
\end{algorithm}

Denote the pairwise and threeway correlations as:
\begin{eqnarray*}
\mu & := & \E[x_1] \\
\Pairs & := & \E[(x_1-\mu) (x_2-\mu)^\t] \\
\Triples & := & \E[(x_1-\mu) \otimes (x_2-\mu) \otimes (x_3-\mu)]\\
\end{eqnarray*}
The dimensions of $\Pairs$ and $\Triples$ are $d^2$ and
$d^3$, respectively. It is convenient to project
$\Triples$ to a matrix as follows:
\begin{eqnarray*}
\Triples(\eta)  & :=  & \E[(x_1-\mu) (x_2-\mu)^\t \dotp{\eta,x_3-\mu}]\\
\end{eqnarray*}
Roughly speaking, we can think of $\Triples(\eta)$ as a reweighing of
a cross covariance (by $\dotp{\eta,x_3-\mu}$).

In addition to $O$ not being identifiable up to permutation, the scale of
each column of $O$ is also not identifiable. To see this, observe the
model over $x_i$ is unaltered if we both rescale any column $O_i$ and
appropriately rescale the variable $h_i$. Without further
assumptions, we can only hope to recover a certain canonical form of
$O$, defined as follows:

\begin{definition}[The Canonical $O$] We say $O$ is in a canonical form if, for
  each $i$, $\sigma_i^2=1$. In particular, the transformation
  $O\leftarrow O \diag(\sigma_1,\sigma_2,\dotsc,\sigma_k)$ (and a
  rescaling of $h$) places $O$ in canonical form, and the distribution
  over $x_1,x_2, x_3,\dotsc$ is unaltered. Observe the canonical
  $O$ is only specified up to the sign of each column (any sign
  change of a column does not alter the variance of $h_i$).
\end{definition}

Recall $\mu_{i,3}$ is the central third moment. Denote the skewness of $h_i$ as:
\[
\gamma_i = \frac{\mu_{i,3}}{\sigma_i^3}
\]
The first result considers the case when the skewness is non-zero.

\begin{theorem}[Independent and skewed factors]
\label{thm:skew}
   We have that:
  \begin{itemize}
  \item (No False Positives) For all $\theta\in \R^k$, Algorithm~\ref{alg:skew}
    returns a subset of the columns of $O$, in a canonical form.
  \item (Exact Recovery) Assume $\gamma_i$ is nonzero for each
    $i$. Suppose $\theta \in \R^k$ is a random vector uniformly
    sampled over the sphere $\sphere^{k-1}$. With probability $1$,
    Algorithm~\ref{alg:skew} returns all columns of $O$, in a
    canonical form.
\end{itemize}
\end{theorem}

The proof of this theorem is a consequence of the following lemma:

\begin{lemma} \label{lemma:pairs-triples}
We have:
\begin{eqnarray*}
\Pairs & = & O \diag(\sigma_1^2,\sigma_2^2,\dotsc,\sigma_k^2) O^\t \\
\Triples(\eta)  & = & O \diag(O^\t\eta)
\diag(\mu_{1,3},\mu_{2,3},\dotsc,\mu_{k,3}) O^\t\\
\end{eqnarray*}
\end{lemma}

The proof of this Lemma is provided in the Appendix.

\begin{proof}[Proof of Theorem~\ref{thm:skew}] The analysis is with respect
  to $O$ it its canonical form. By the full rank assumption, $U^\t
  \Pairs U$, which is a $k \times k$ matrix, is full rank; hence, the
  whitening step is possible.  By
  construction:
\begin{eqnarray*}
\I &=&W^\t \Pairs W\\
& = & W^\t  O \diag(\sigma_1^2,\sigma_2^2,\dotsc,\sigma_k^2)  O^\t W \\
& = & (W^\t  O) (W^\t  O)^\t\\
& := & M M^\t
\end{eqnarray*}
where $M := W^\t O$. Hence, $ M$ is a $k\times k$
orthogonal matrix.

Observe:
\begin{eqnarray*}
W^\t \Triples(W\theta) W
& = &
W^\t  O \diag( O^\t W \theta)
\diag(\gamma_1,\gamma_2,\dotsc,\gamma_k)  O^\t W\\
 & = &
 M \diag( M^\t \theta)
\diag(\gamma_1,\gamma_2,\dotsc,\gamma_k)  M^\t\\
\end{eqnarray*}
Since $ M$ is an orthogonal matrix, the above is a (not
necessarily unique) singular value decomposition of $W^\t
\Triples(W\theta) W$. Denote the standard basis as $e_1,e_2, \dotsc
e_k$. Observe that $ M e_1, \dotsc M e_k$ are singular vectors. In other words, $W^\t
O_1, \dotsc W^\t O_k$ are singular vectors, where $ O_i$ is the $i$-th
column of $ O$.

An SVD uniquely determines all singular vectors (up to sign) which
have unique singular values.  The diagonal of the
matrix $\diag( M^\t \theta) \diag(\gamma_1,\gamma_2,\dotsc,\gamma_k)$
is the vector $\diag(\gamma_1,\gamma_2,\dotsc,\gamma_k) M^\t
\theta$. Also, since $M$ is a rotation matrix, the distribution of
$M\theta$ is also uniform on the sphere. Thus, if $\theta$ is
uniformly sampled over the sphere, then every singular value will be
nonzero (and distinct) with probability $1$.  Finally, for the
reconstruction, we have
\[
W (W^\t
W)^{-1} M e_i = W (W^\t
W)^{-1} W^\t O_i = O_i ,
\]
since $W (W^\t W)^{-1} W^\t$ is a projection
operator (and the range of $W$ and $O$ are identical).
\end{proof}

\begin{remark}[Finding $\textrm{Range}(\Pairs)$ efficiently]
\label{remark:range}
   Suppose $\Theta\in \R^{d \times
    k}$ is a random matrix with entries sampled independently from a
  standard normal. Set $U=\Pairs \Theta$. Then, with probability $1$,
  $\textrm{Range}(U)=\textrm{Range}(\Pairs)$.
\end{remark}

\begin{remark}[No false positives]
  Note that if the skewness is $0$ for some $i$
  then ECA will not recover the corresponding column. However, the
  algorithm does succeed for those directions in which the skewness is
  non-zero. This guarantee also provides the practical freedom to run
  the algorithm with multiple different directions $\theta$, since we
  need only to find unique singular vectors (which may be easier to
  determine by running the algorithm with different choices for
  $\theta$).
\end{remark}

\begin{remark}[Estimating the skewness]
\label{remark:skew}
  It is straight forward to estimate the
  skewness corresponding to any column of $O$. Suppose $\lambda$ is
  some unique singular vector (up to sign) found in step 3 of ECA
  (which was used to construct some column $O_i$), then:
\[
\gamma_i = \lambda^\t W^\t \Triples(W \lambda) W \lambda
\]
is the corresponding skewness for $O_i$. This follows from the
proof, since $\lambda$ corresponds to some singular vector $Me_i$ and:
\[
(Me_i)^\t M \diag( M^\t Me_i)
\diag(\gamma_1,\gamma_2,\dotsc,\gamma_k)  M^\t Me_i = \gamma_i
\]
using that $M$ is an orthogonal matrix.
\end{remark}

\subsection{Independent and Kurtotic Latent Factors}

\begin{algorithm} [t]
\caption{ECA; with kurtotic factors} \label{alg:kurt}
\begin{algorithmic}
\STATE {\bf Input:}
vectors $\theta, \theta' \in \R^k$; the moments $\Pairs$ and $\Quad(\eta,\eta')$
\begin{enumerate}
\STATE {\bf Dimensionality Reduction:}
Find a matrix $U\in \R^{d \times k}$ such that
\[
\textrm{Range}(U)=\textrm{Range}(\Pairs).
\]
\STATE {\bf Whiten:} Find $V\in \R^{k \times k}$ so
$V^\t (U^\t \Pairs U) V $
is the $k\times k$ identity matrix. Set:
\[
W =  UV
\]
\STATE {\bf SVD:}  Let $\Lambda$ be the set of
(left) singular vectors, with \emph{unique} singular
values, of
\[
W^\t \Quad(W \theta,W \theta') W
\]
\STATE {\bf Reconstruct:} Return the set $\wh O$:
\[
\wh O = \{ \ (W^+)^\t \lambda \ : \lambda \in \Lambda\}
\]
where $W^+$ is the pseudo-inverse (see Eq~\ref{eq:pseudo}).
\end{enumerate}
\end{algorithmic}
\end{algorithm}

Define the following matrix:
\begin{multline*}
\Quad(\eta,\eta') :=
\E[(x_1-\mu) (x_2-\mu)^\t \dotp{\eta,x_3-\mu}\dotp{\eta',x_4-\mu}]\\
 - (\eta^\t \Pairs \eta') \Pairs - (\Pairs \eta) (\Pairs \eta')^\t - (\Pairs \eta') (\Pairs \eta)^\t
\end{multline*}
This is a subspace of the fourth moment tensor.

Recall $\mu_{i,4}$ is the central fourth moment. Denote the excess kurtosis of $h_i$ as:
\[
\kappa_i = \frac{\mu_{i,4}}{\sigma_i^4}-3
\]
For Gaussian distributions, recall the kurtosis is $3$, and so the
excess kurtosis is $0$.  This function is also common in the source
separation approaches~\citep{ICAbook}~\footnote{Their
  algebraic method require more effort due to the additive noise and
  the lack of exchangeability. Here, the exchangeability assumption
  simplifies the approach and allows us to address models with
  non-additive noise (as in the Poisson count model discussed in the
  Section~\ref{sec:models}.}.

In settings where the latent factors are not skewed,
we may hope  that they are differentiated from a Gaussian distribution  due to
their fourth order moments. Here, Algorithm~\ref{alg:kurt} is applicable:

\begin{theorem}[Independent and kurtotic factors]
\label{thm:kurt}
   We have that:
  \begin{itemize}
  \item (No False Positives) For all $\theta,\theta' \in \R^k$, Algorithm~\ref{alg:kurt}
    returns a subset of the columns of $O$, in a canonical form.
  \item (Exact Recovery) Assume $\kappa_i$ is nonzero for each $i$. Suppose $\theta,\theta' \in \R^k$ are random vectors
     uniformly and independently sampled over the sphere $\sphere^{k-1}$. With probability $1$,
    Algorithm~\ref{alg:kurt} returns all the columns of $O$, in
    a canonical form.
\end{itemize}
\end{theorem}

\begin{remark}[Using both skewed and kurtotic ECA]
  Note that both algorithms never
  incorrectly return columns. Hence, if for every $i$, \emph{either} the
  skewness or the excess kurtosis is nonzero, then by
  running both algorithms we will recover $O$.
\end{remark}

The proof of this theorem is a consequence of the following lemma:

\begin{lemma} \label{lemma:kurt}
We have:
\begin{eqnarray*}
\Quad(\eta,\eta')   =
O \diag(O^\t\eta) \diag(O^\t\eta')
\diag(\mu_{1,4}-3\sigma_1^4,\mu_{2,4}-3\sigma_2^4,\dotsc,\mu_{k,4}-3\sigma_k^4)
O^\t
\end{eqnarray*}
\end{lemma}

The proof of this Lemma is provided in the Appendix.

\begin{proof}[Proof of Theorem~\ref{thm:kurt}]
The distinction from the argument in Theorem~\ref{thm:skew} is that:
\begin{eqnarray*}
W^\t \Quad(W\theta,W\theta')) W
& = &
W^\t  O \diag( O^\t W \theta) \diag( O^\t W \theta')
\diag(\kappa_1,\kappa_2,\dotsc,\kappa_k)  O^\t W\\
& = &
M \diag( M^\t \theta) \diag( M^\t \theta')
\diag(\kappa_1,\kappa_2,\dotsc,\kappa_k)  M^\t\\
\end{eqnarray*}
The remainder of the argument follows that of the proof of Theorem~\ref{thm:skew}.
\end{proof}

\subsection{Latent Dirichlet Allocation}

Now let us turn to the case where $h$ has a Dirichlet density, where,
each $h_i$ is not sampled independently. Even though the distribution
on $h$ is the product of $\hidden_i^{\alpha_1-1},\dotsc
\hidden_i^{\alpha_k-1} $, the $h_i$'s are not independent due to the
constraint that $h$ lives on the simplex. These dependencies suggest
a modification for the moments to be used in ECA, which we now
provide.

Suppose $\alpha_0$ is known. Recall that $\alpha_0 := \alpha_1 +
\alpha_2 + \dotsb + \alpha_k $ (the sum of the ``pseudo-counts'').
Knowledge of $\alpha_0$ is significantly weaker than having full
knowledge of the entire parameter vector $\alpha$. A common practice
is to specify the entire parameter vector $\alpha$ in a homogeneous
manner, with each component being identical (see
\cite{steyvers2006probabilistic}). Here, we need only specify the sum,
which allows for arbitrary inhomogeneity in the prior.

Denote the mean as
\[
\mu = \E[x_1] 
\]
Define a modified second moment as
\[
\Pairs_{\alpha_0} : = \E[x_1 x_2 ^\t] -
\frac{\alpha_0}{\alpha_0+1}
\mu \mu ^\t  \\
\]
and  a modified third moment as
\begin{multline*}
\Triples_{\alpha_0}(\eta)
: = \E[x_1 x_2^\t \dotp{\eta,x_3}]
-\frac{\alpha_0}{\alpha_0+2}
\Bigl(\E[x_1 x_2 ^\t]  \eta \mu^\t
+\mu \eta^\t \E[x_1 x_2 ^\t]  + \dotp{\eta,\mu} \E[x_1 x_2 ^\t]
\Bigr)
\\
+ \frac{2\alpha_0^2}{(\alpha_0+2)(\alpha_0+1)}
\dotp{\eta,\mu} \mu \mu^\t
\end{multline*}

\begin{remark}[Central vs Non-Central Moments]
 In the limit as $\alpha_0 \rightarrow
0$, the Dirichlet model degenerates so that, with probability $1$, only one coordinate of $h$ equals $1$
and the rest are $0$ (\emph{e.g.}, each document is about $1$ topic). Here, we
limit to  non-central moments:
\[
\lim_{\alpha_0 \rightarrow 0} \Pairs_{\alpha_0} = \E[x_1 x_2^\t ]
\quad
\lim_{\alpha_0 \rightarrow 0} \Triples_{\alpha_0}(\eta) = \E[x_1 x_2^\t \dotp{\eta,x_3}]
\]
In the other extreme, the behavior limits to the central moments:
\[
\lim_{\alpha_0 \rightarrow \infty} \Pairs_{\alpha_0} = \E[(x_1-\mu) (x_2-\mu)^\t ]
\quad
\lim_{\alpha_0 \rightarrow \infty} \Triples_{\alpha_0}(\eta) = \E[(x_1-\mu) (x_2-\mu)^\t \dotp{\eta,(x_3-\mu)}]
\]
(to prove the latter claim, expand the central moment and use that, by
exchangeability, $\E[x_1 x_2^\t ] = \E[x_2 x_3^\t ] = \E[x_1 x_3^\t ]$).
\end{remark}

\begin{algorithm} [t]
\caption{ECA for latent Dirichlet allocation} \label{alg:lda}
\begin{algorithmic}
\STATE {\bf Input:}
a vector $\theta \in \R^k$; the moments $\Pairs_{\alpha_0}$ and $\Triples_{\alpha_0}$
\begin{enumerate}
\STATE {\bf Dimensionality Reduction:}
Find a matrix $U\in \R^{d \times k}$ such that
\[
\textrm{Range}(U)=\textrm{Range}(\Pairs_{\alpha_0}).
\]
(See Remark~\ref{remark:range} for a fast procedure.)
\STATE {\bf Whiten:} Find $V\in \R^{k \times k}$ so
$V^\t (U^\t \Pairs_{\alpha_0} U) V $
is the $k\times k$ identity matrix. Set:
\[
W =  UV
\]
\STATE {\bf SVD:}  Let $\Lambda$ be the set of
(left) singular vectors, with \emph{unique} singular
values, of
\[
W^\t \Triples_{\alpha_0} (W \theta) W
\]
\STATE {\bf Reconstruct and Normalize:} Return the set $\wh O$:
\[
\wh O = \left\{ \
\frac{(W^+)^\t \lambda}{\vec{1} ^\t (W^+)^\t \lambda} \
: \lambda \in \Lambda\right\}
\]
where $\vec{1}\in \R^d$ is a vector of all ones and $W^+$ is the pseudo-inverse (see Eq~\ref{eq:pseudo}).
\end{enumerate}
\end{algorithmic}
\end{algorithm}

Our main result here shows that ECA recovers both the topic matrix $O$, up
to a permutation of the columns
(where each column represents a probability distribution over words
for a given topic) and the parameter vector $\alpha$, using only knowledge of
$\alpha_0$ (which, as discussed earlier, is a significantly less
restrictive assumption than tuning the entire parameter
vector). Also, as discussed in Remark~\ref{remark:lda_general}, the
method applies to cases where $x_v$ is not a multinomial
distribution.

\begin{theorem}[Latent Dirichlet Allocation]
\label{thm:lda}
  We have that:
  \begin{itemize}
  \item (No False Positives) For all $\theta \in \R^k$,
    Algorithm~\ref{alg:lda} returns a subset of the columns of $O$.
  \item (Topic Recovery) Suppose $\theta \in \R^k$ is a random vector
    uniformly sampled over the sphere $\sphere^{k-1}$. With
    probability $1$, Algorithm~\ref{alg:lda} returns all columns of
    $O$.
  \item (Parameter Recovery) We have that:
\[
\alpha = \alpha_0 (\alpha_0+1) O^{+} \Pairs_{\alpha_0} (O^+)^\t \, \vec{1}
\]
where $\vec{1}\in \R^k$ is a vector of all ones.
  \end{itemize}
\end{theorem}

The proof is a consequence of the following lemma:

\begin{lemma} \label{lemma:lda_pairs_triples}
We have:
\[
\Pairs_{\alpha_0} = \frac1{(\alpha_0+1)\alpha_0} O \diag(\alpha) O^\t
\]
and
\[
\Triples_{\alpha_0}(\eta) = \frac{2}{(\alpha_0+2)(\alpha_0+1)\alpha_0} O
\diag(O^\t \eta) \diag(\alpha) O^\t
\]
\end{lemma}

The proof of this Lemma is provided in the Appendix.

\begin{proof}[Proof of Theorem~\ref{thm:lda}]
Note that with the following rescaling of columns:
\[
\tilde O= \frac1{\sqrt{(\alpha_0+1)\alpha_0}} O
\diag(\sqrt{\alpha_1},\sqrt{\alpha_2},\dotsc,\sqrt{\alpha_k})
\]
we have that $h$ is in canonical form (\emph{i.e.}, the variance of each $h_i$
is 1). The remainder of the proof is identical to that of
Theorem~\ref{thm:skew}. The only modification is that we simply
normalize the output of Algorithm~\ref{alg:skew}. Finally, observe
that claim for estimating $\alpha$ holds due to the functional
form of $\Pairs_{\alpha_0}$.
\end{proof}

\begin{remark}[Limiting behaviors] ECA seamlessly blends between the single
  topic model $(\alpha_0\to 0)$ of \cite{AHKparams} and the skewness based ECA,
  Algorithm~\ref{alg:skew} $(\alpha_0\to \infty)$.  In the single topic case,
  \cite{AHKparams} provide eigenvector based algorithms. This
  work shows that two SVDs suffice for parameter recovery.
\end{remark}

\begin{remark}[Skewed and Kurtotic ECA for LDA]
  We conjecture that the fourth
  moments can be utilized in the Dirichlet case such that the
  resulting algorithm limits to the kurtotic based ECA, when
  $\alpha_0\rightarrow \infty$.   Furthermore, the mixture of
  Poissions model discussed in Section~\ref{sec:models} provides a
  natural alternative to the LDA model in this regime.
\end{remark}

\begin{remark}[The Dirichlet model, more generally]\label{remark:lda_general} It is not necessary that we have a multinomial
  distribution on $x_v$, so long as $\E[x_v|h] = Oh$. In some
  applications, it might be natural for the observations to come from
  a different distribution (say $x_v$ may represent pixel intensities
  in an image or some other real valued quantity). For this case,
  where $h$ has a Dirichlet prior (and where $x_v$ may not be
  multinomial), ECA still correctly recovers the columns of
  $O$. Furthermore, we need not normalize; the set $ \left\{ (W^+)^\t
    \lambda : \lambda \in \Lambda\right\} $ recovers $O$ in a canonical
  form.
\end{remark}

\subsection{The Multi-View Extension}

\begin{algorithm} [t]
\caption{ECA; the multi-view case} \label{alg:multi}
\begin{algorithmic}
\STATE {\bf Input:}
vector $\theta \in \R^k$; the moments $\Pairs_{v,v'}$ and $\Triples_{132}(\eta)$

\begin{enumerate}
\STATE {\bf Project views $1$ and $2$:}
Find matrices $A\in \R^{k\times d_1}$ and $B\in\R^{k\times d_2}$  such
that $A \Pairs_{12}B^\t$ is invertible. Set:
\begin{eqnarray*}
\wt \Pairs_{12} &:=& A\Pairs_{12}B^\t \\
\wt \Pairs_{31} &:=&\Pairs_{31}A^\t\\
\wt \Pairs_{32} &:=&\Pairs_{32} B^\t\\
\wt \Triples_{132}(\eta) & := & A \Triples_{132}(\eta) B^\t
\end{eqnarray*}
(See Remark~\ref{remark:AB} for a fast procedure.)

\STATE {\bf Symmetrize:} Reduce to a single view:
\begin{eqnarray*}
\Pairs_3 & := & \wt \Pairs_{31} (\wt \Pairs_{12}^\t)^{-1} \wt \Pairs_{23}\\
\Triples_3 (\eta) & := & \wt \Pairs_{32} (\wt \Pairs_{12})^{-1}
\wt \Triples_{132}(\eta)
(\wt \Pairs_{12})^{-1} \wt \Pairs_{13}
\end{eqnarray*}

\STATE {\bf Estimate $O_3$ with ECA:}  Call Algorithm~\ref{alg:skew}, with $\theta$, $\Pairs_3$, and $\Triples_3(\eta)$.
\end{enumerate}
\end{algorithmic}
\end{algorithm}

Rather than $O$ being identical for each $x_v$, suppose for each $v \in
\{1,2,3,4,\dotsc\}$ there exists an $O_v\in \R^{d_v\times k}$ such that
\[
\E[x_v|h] = O_v h
\]
For $v \in \{1,2,3\}$, define
\begin{eqnarray*}
\Pairs_{v,v'} & := & \E[(x_v-\mu) (x_v'-\mu)^\t] \\
\Triples_{132}(\eta)  & :=  & \E[(x_1-\mu) (x_2-\mu)^\t \dotp{\eta,x_3-\mu}]
\end{eqnarray*}
We use the notation $132$ to stress that $\Triples_{132}(\eta)$ is a
$d_1 \times d_2$ sized matrix.

\begin{lemma} \label{lemma:multi}
For $v \in \{1,2,3\}$,
\begin{eqnarray*}
\Pairs_{v,v'} & = & O_v \diag(\sigma_1^2,\sigma_2^2,\dotsc,\sigma_k^2) O_v'^\t \\
\Triples_{132}(\eta)  & = & O_1 \diag(O_3^\t\eta)
\diag(\mu_{1,3},\mu_{2,3},\dotsc,\mu_{k,3}) O_2^\t\\
\end{eqnarray*}
\end{lemma}

The proof for Lemma~\ref{lemma:multi}  is analogous to those in Appendix~\ref{sec:analysis}.

These functional forms make deriving an SVD based algorithm more
subtle. Using the methods in \cite{AHKparams}, eigenvector based
method are straightforward to derive. However, SVD based algorithms
are preferred due to their greater simplicity. The following lemma
shows how the symmetrization step in the algorithm makes this possible.

\begin{lemma} \label{corollary:multi}
For $\Pairs_3$ and $\Triples_3(\eta)$ defined in
Algorithm~\ref{alg:multi}, we have:
\begin{eqnarray*}
\Pairs_3 & = & O_3 \diag(\sigma_1^2,\sigma_2^2,\dotsc,\sigma_k^2)
O_3^\t \\
\Triples_3(\eta)  & = & O_3 \diag(O_3^\t\eta)
\diag(\mu_{1,3},\mu_{2,3},\dotsc,\mu_{k,3}) O_3^\t
\end{eqnarray*}
\end{lemma}

\begin{proof}
  Without loss of generality, suppose $O_v$ are in canonical form (for
  each $i$, $\sigma_i^2=1$). Hence, $A \Pairs_{12}B^\t = A O_1 (B
  O_2)^\t $. Hence, $A O_1$ and $B O_2$ are invertible.  Note that:
\[
\Pairs_{31} A^\t (B \Pairs_{21}A^\t)^{-1} B \Pairs_{23}
= O_3 O_1^\t A^\t (B O_2 O_1^\t A^\t)^{-1} B O_2 O_3^\t = O_3 O_3^\t
\]
which proves the first claim. The proof of the second claim is analogous.
\end{proof}

Again, we say that all $O_v$ are in a canonical form if, for
each $i$, $\sigma_i^2=1$.

\begin{theorem}[The multi-view case]
\label{thm:multi}
  We have:
  \begin{itemize}
  \item (No False Positives) For all $\theta\in \R^k$, Algorithm~\ref{alg:multi}
    returns a subset of $O_3$, in a canonical form.
  \item (Exact Recovery) Assume that $\gamma_i$ is nonzero for each
    $i$.  Suppose $\theta \in \R^k$ is a random vector uniformly
    sampled over the sphere $\sphere^{k-1}$. With probability $1$,
    Algorithm~\ref{alg:multi} returns all columns of $O_3$, in a
    canonical form.
\end{itemize}
\end{theorem}

\begin{proof}[Proof of Theorem~\ref{thm:multi}]
The proof is identical to that of Theorem~\ref{thm:skew}.
\end{proof}

\begin{remark}[Simpler algorithms for HMMs]
  \cite{MR06,AHKparams} provide
  eigenvector based algorithms for HMM parameter estimation. These
  results show that we can achieve parameter estimation with only two
  SVDs (see~\cite{AHKparams} for the reduction of an HMM to the
  multi-view setting). The key idea is the symmetrization that reduces
  the problem to a single view.
\end{remark}

\begin{remark}[Finding $A$ and $B$]\label{remark:AB}
   Suppose $\Theta,\Theta' \in \R^{d \times k}$
  are random matrices with entries sampled independently from a
  standard normal. Set $A =\Pairs_{1,2} \Theta$ and $B =\Pairs_{2,1}
  \Theta'$. With probability $1$,
  $\textrm{Range}(A)=\textrm{Range}(O_1)$ and
  $\textrm{Range}(B)=\textrm{Range}(O_2)$, and the invertibility
  condition will be satisfied (provided that $O_1$ and $O_2$ are full
  rank).
\end{remark}

\section{Sample Complexity}\label{sec:samp_comp}


\begin{algorithm} [t]
\caption{Empirical ECA for LDA} \label{alg:LDA_empirical}
\begin{algorithmic}
\STATE {\bf Input:}
an integer $k$; an integer $N$; vector $\theta \in \R^k$; the sum $\alpha_0$

\begin{enumerate}
\STATE {\bf Find Empirical Averages:} With $N$ independent samples (of documents),
compute the empirical first, second, and third moments. Then compute
the empirical moments $\wh \Pairsa$ and $\wh \Triplesa(\eta)$.
\STATE {\bf Whiten:} Let $\h W  = A \Sigma^{-1/2} \in \R^{d\times k}$ where
$A \in \R^{d\times k}$ is the matrix of
the orthonormal left singular vectors of $\wh \Pairsa$, corresponding
to the largest $k$ singular values, and $\Sigma \in \R^{k\times k}$ is the
corresponding diagonal matrix of the $k$ largest singular values.
\STATE {\bf SVD:}  Let $\{\hat v_1,\hat v_2,\ldots \hat v_k\}$ be the set of
(left) singular vectors of
\[
\h W^\t \wh\Triplesa(\h W \theta) \h W
\]
\STATE {\bf Reconstruct and Scale:} Return the set $\{\h O_1,\h
O_2,\ldots \h O_k\}$ where
\begin{eqnarray*}
\h Z_i &= &\frac{2}
{(\alpha_0+2) (\h W \hat v_i )^\t 
\wh\Triplesa(\h W \h v_i) \h W \hat v_i}\\
\\
\h O_i & = & \frac{1}{\h Z_i} \ (\h W^+)^\t \hat v_i
\end{eqnarray*}
\\(See Remark~\ref{remark:clip} for a procedure which explicitly
normalizes $\h O_i$.)
\end{enumerate}
\end{algorithmic}
\end{algorithm}

Let us now provide an efficient algorithm utilizing samples from documents,
rather than exact statistics. The following theorem shows that the
empirical version of ECA returns accurate estimates of the
topics. Furthermore, each run of the algorithm succeeds with
probability greater than $3/4$ so the algorithm may be repeatedly
run. Primarily for theoretical analysis, Algorithm~\ref{alg:LDA_empirical}
uses a rescaling procedure (rather than explicitly normalizing the
topics, which would involve some thresholding procedure; see
Remark~\ref{remark:clip}). 

\begin{theorem}[Sample Complexity for LDA]
\label{thm:samples}
  Fix $\delta \in (0,1)$.  Let $\pmin = \min_i
  \frac{\alpha_i}{\alpha_0}$ and let $\sigma_k(O)$ denote the smallest
  (non-zero) singular value of $O$.  Suppose that we obtain 
  $N\geq
  \left( \frac{(\alpha_0+1 )(6 + 6\sqrt{\ln(3/\delta)})}{\pmin
      \sigma_k(O)^2}\right)^2$ 
  independent samples of $x_1, x_2, x_3$
  in the LDA model. With probability greater than $1-\delta$, the
  following holds: for $\theta \in \R^k$ sampled uniformly sampled
  over the sphere $\sphere^{k-1}$, with probability greater than
  $3/4$, Algorithm~\ref{alg:LDA_empirical} returns a set $\{\h O_1,\h
  O_2, \ldots \h O_k\}$ such that there exists a permutation $\sigma$
  of $\{1,2,\ldots k\}$ (a permutation of the columns) so that for all
  $i \in \{1,2,\ldots k\}$
\[
\| O_i - \h O_{\sigma(i)}\|_2 \leq c \ 
\frac{ (\alpha_0+1)^2  k^3}{\pmin^{2} \sigma_k(O)^{3} }   \ \left( \frac{1 +
  \sqrt{\ln(1/\delta)}}{\sqrt{N}}\right)
\]
where $c$ is a universal
constant.
\end{theorem}

\vspace*{0.1in}
\begin{remark}[Normalizing  and $\ell_1$ accuracy]
\label{remark:clip}
  An alternative procedure would be to just explicitly
  normalize $\hat O_i$. If $d$ large, to do this
  robustly, one should first set to $0$ the smallest elements and then
  normalize. The reason for clipping the smallest elements is related to
  obtaining low $\ell_1$ error.

  Our theorem currently guarantees $\ell_2$
  norm accuracy of each column. Another natural error measure for
  probability distributions is the $\ell_1$ error (the total variation
  error). Ideally, we would like the $\ell_1$ error to be small with a
  number of samples does not depend on the dimension $d$ (\emph{e.g.}, the
  size of the vocabulary). Unfortunately, in general, this is not
  possible. For example, in the simplest case where $k=1$ (\emph{i.e.}, every
  document is about the same topic), then this amounts to estimating
  the distributions over words for this topic; in other words, we must
  estimate a distribution over $d$, which may require $\Omega(d)$
  samples to obtain some fixed target $\ell_1$-error. However, this
  situation occurs only when the target distribution is near to
  uniform. If instead, for each topic, say most of the probability
  mass is contained within the most frequent $d_{\textrm{effective}}$
  words (for that topic), then it is possible to translate our
  $\ell_2$ error guarantee into an $\ell_1$ guarantee (in terms of
  $d_{\textrm{effective}}$).
\end{remark}

\section{Discussion: Sparsity}\label{sec:discussion}

Note that sparsity considerations have not entered into
our analysis. Often, in high dimensional statistics,
notions of sparsity are desired as this generally decreases the sample
size requirements (often at an increased computational burden).

Here, while these results have no explicit dependence on the sparsity
level, sparsity is helpful in that it does implicitly affect the
skewness (and the whitening) , which determines the sample
complexity. As the model becomes less sparse, the skewness tends to
$0$.  In particular, for the case of LDA, as $\alpha_0 \rightarrow
\infty$ note that error increases (see Theorem~\ref{thm:samples}).

Perhaps surprisingly, the sparsity level has no direct impact on the
computational requirements of a ``plug-in'' empirical algorithm (beyond the
linear time requirement of reading the data in order to construct the empirical
statistics).

\subsection*{Acknowledgements}
We thank Kamalika Chaudhuri, Adam Kalai, Percy Liang, Chris Meek,
David Sontag, and Tong Zhang for many invaluable insights. We also
give warm thanks to Rong Ge for sharing early insights and their
preliminary results (in~\cite{Arora:1439928}) into this problem with
us.

\bibliography{lda}
\bibliographystyle{plainnat}

\appendix

\section{Analysis with Independent Factors} \label{sec:analysis}

\begin{lemma}[Hidden state moments] \label{lemma:mixed-moments}
Let $z = h - \E[h]$. For any vectors
$u, v \in \R^k$,
\begin{eqnarray*}
\E[ z z^\t ] 
&= &
\diag(\sigma_1^2,\sigma_2^2,\dotsc,\sigma_k^2) \\
\E[ z z^\t \dotp{u,z} ] 
& = &
\diag(u) \diag(\mu_{i,3},\mu_{2,3},\dotsc,\mu_{k,3}) 
\end{eqnarray*}
and
\begin{multline*}
\E[ z z^\t \dotp{u,z} \dotp{v,z} ] 
= 
\diag(u) \diag(v) \diag(\mu_{1,4}-3 \sigma_1^4,\mu_{2,4}-3 \sigma_2^4,\dotsc,\mu_{k,4}-3 \sigma_k^4) \\
+(u^\t \E[ z z^\t ] v) \E[ z z^\t ] +  (\E[ z z^\t ]u) (\E[ z z^\t ] v)^\t +
(\E[ z z^\t ]v) (\E[ z z^\t ] u)^\t
\end{multline*}
\end{lemma}

\begin{proof}
Let $a$, $b$, $u$ and $v$ be vectors. Since the $\{ z_t \}$ are independent and
have mean zero, we have:
\begin{equation*}
\E[ \dotp{a,z} \dotp{b,z} ]
= \E\biggl[
\biggl( \sum_{i=1}^k a_i z_i \biggr)
\biggl( \sum_{i=1}^k b_i z_i \biggr)
\biggr]
= \sum_{i=1}^k a_i b_i \E[z_i^2]
= \sum_{i=1}^k a_i b_i \sigma_i^2
\end{equation*}
and
\begin{equation*}
\E[ \dotp{a,z} \dotp{b,z} \dotp{u,z} ]
= \E\biggl[
\biggl( \sum_{i=1}^k a_i z_i \biggr)
\biggl( \sum_{i=1}^k b_i z_i \biggr)
\biggl( \sum_{i=1}^k u_i z_i \biggr)
\biggr]
= \sum_{i=1}^k a_i b_i u_i \E[z_i^3]
= \sum_{i=1}^k a_i b_i u_i \mu_{i,3}
.
\end{equation*}
For the final claim, let us compute the diagonal and non-diagonal
entries separately. First,
\begin{eqnarray*}
\E[ z_i z_i \dotp{u,z} \dotp{v,z} ] & = & \E[ \sum_{j,k} u_j v_k z_i z_i z_j z_k ]\\
& = & u_i v_i \E[z_i^4] + \sum_{j\neq i} u_j v_j \E[z_i^2] \E[z_j^2]\\
& = & u_i v_i \mu_{i,4} + \sigma_i^2 \sum_{j\neq i} u_j v_j \sigma_j^2
\\
& = & u_i v_i \mu_{i,4} - u_i v_i (\sigma_i^2)^2  +\sigma_i^2\sum_{j} u_j v_j
\sigma_j^2 \\
& = & u_i v_i \mu_{i,4} - u_i v_i (\sigma_i^2)^2  + (u^\t \E[ z z^\t ] v ) \sigma_i^2
\end{eqnarray*}
For $j\neq i$
\begin{eqnarray*}
\E[ z_i z_j \dotp{u,z} \dotp{v,z} ]
 & = &  \E[ \sum_{k,l} u_k v_l z_i z_j z_k z_l ]\\
 & = &  u_i v_j \E[z_i^2 z_j^2 ] +u_j v_i \E[z_i^2 z_j^2 ]\\
 & = &  u_i v_j \sigma_i^2 \sigma_j^2 +u_j v_i \sigma_i^2 \sigma_j^2\\
 & = &  [\E[ z z^\t ] u]_i  [\E[ z z^\t ] v]_j+[\E[ z z^\t ] u]_j  [\E[ z z^\t ] v]_i
\end{eqnarray*}
The proof is completed by noting the $(i,j)$-th components of
$\E[ z z^\t \dotp{u,z} \dotp{v,z} ] $ agree with the above moment expressions.
\end{proof}

The proofs of Lemmas~\ref{lemma:pairs-triples} and ~\ref{lemma:kurt} follow.

\begin{proof}[Proof of Lemmas~\ref{lemma:pairs-triples} and
 \ref{lemma:kurt}]
By the conditional independence of $\{x_1, x_2, x_3\}$ given $h$,
\[
\E[x_1] = O\E[\hidden]
\]
and
\begin{align*}
\E[(x_1-\mu) (x_2-\mu)^\t] &= \E[\E[(x_1-\mu) (x_2-\mu)^\t|h]] \\
&= \E[\E[(x_1-\mu)|h] \E[(x_2-\mu)^\t|h]] \\
&= O \E[(h-\E[h]) (h-\E[h])^\t] O^\t \\
&=O \diag(\sigma_1^2,\sigma_2^2,\dotsc,\sigma_k^2) O^\t
\end{align*}
by Lemma~\ref{lemma:mixed-moments}.

Similarly, the $(i,j)$-th entry of $\Triples(\eta)$ is
\begin{align*}
\E\bigl[\dotp{e_i,x_1-\mu} \dotp{e_j,x_2-\mu} \dotp{\eta,x_3-\mu} \bigr]
& = \E\bigl[ \E[\dotp{e_i,x_1-\mu} \dotp{e_j,x_2-\mu} \dotp{\eta,x_3-\mu} |
h] \bigr] \\
& = \E\bigl[
\E[\dotp{e_i,x_1-\mu} | h]
\cdot \E[\dotp{e_j,x_2-\mu} | h]
\cdot \E[\dotp{\eta,x_3-\mu} | h]
\bigr] \\
& = \E\bigl[
\dotp{e_i,O(h-\E[h])}
\dotp{e_j,O(h-\E[h])}
\dotp{\eta,O(h-\E[h])}
\bigr] \\
& = \E\bigl[
\dotp{O^\t e_i,h-\E[h]}
\dotp{O^\t e_j,h-\E[h]}
\dotp{O^\t \eta,h-\E[h]}
\bigr] \\
& = e_i^\t O \diag(O^\t \eta) \diag(\mu_{i,3},\mu_{2,3},\dotsc,\mu_{k,3}) O^\t e_j
.
\end{align*}

The proof for $\Quad(\eta,\eta')$ is analogous.
\end{proof}

The proof for Lemma~\ref{lemma:multi}  is analogous to the above proofs.

\section{Analysis with Dirichlet Factors} 

We first provide the functional forms of the first, second, and third
moments. With these, we prove Lemma~\ref{lemma:lda_pairs_triples}.

\subsection{Dirichlet moments}

\begin{lemma}[Dirichlet moments]
We have:
\[ \E[\hidden \otimes \hidden] = \frac1{(\alpha_0+1)\alpha_0} \bigl( \diag(\alpha) +
\alpha\alpha^\t \bigr) \]
and
\begin{multline*}
\E[\hidden \otimes \hidden \otimes \hidden]
= \frac1{(\alpha_0+2)(\alpha_0+1)\alpha_0}
\Bigl( \alpha \otimes \alpha \otimes \alpha
+ \sum_{i=1}^k \alpha_i \bigl( e_i \otimes e_i \otimes \alpha \bigr)
+ \sum_{i=1}^k \alpha_i \bigl( \alpha \otimes e_i \otimes e_i \bigr)
\\
+ \sum_{i=1}^k \alpha_i \bigl( e_i \otimes \alpha \otimes e_i \bigr)
+ 2 \sum_{i=1}^k \alpha_i \bigl( e_i \otimes e_i \otimes e_i \bigr)
\Bigr) .
\end{multline*}
Hence, for $v \in \R^k$,
\begin{multline*}
\E[(\hidden \otimes \hidden) \dotp{v,\hidden}]
= \frac1{(\alpha_0+2)(\alpha_0+1)\alpha_0}
\Bigl( \dotp{v,\alpha} \alpha \alpha^\t
+ \diag(\alpha) v\alpha^\t
+ \alpha v^\t \diag(\alpha)
\\
+ \dotp{v,\alpha} \diag(\alpha)
+ 2 \diag(v) \diag(\alpha)
\Bigr)
\end{multline*}
\end{lemma}

\begin{proof}

First, let us specify the following scalar moments.

{\bf Univariate moments: }
Fix some $i \in [k]$, and let $\alpha' := \alpha + p \cdot e_i$ for some
positive integer $p$.
Then
\begin{align*}
\E[\hidden_i^p]
& = \frac{Z(\alpha')}{Z(\alpha)} \\
& = \frac{\Gamma(\alpha_i+p)} {\Gamma(\alpha_i)}
\cdot \frac{\Gamma(\alpha_0)} {\Gamma(\alpha_0+p)} \\
& = \frac{(\alpha_i+p-1) (\alpha_i+p-2) \dotsb \alpha_i}
{(\alpha_0+p-1) (\alpha_0+p-2) \dotsb \alpha_0}
.
\end{align*}
In particular,
\begin{align*}
\E[\hidden_i] & = \frac{\alpha_i}{\alpha_0} \\
\E[\hidden_i^2] & = \frac{(\alpha_i+1)\alpha_i}
{(\alpha_0+1)\alpha_0} \\
\E[\hidden_i^3] & = \frac{(\alpha_i+2)(\alpha_i+1)\alpha_i}
{(\alpha_0+2)(\alpha_0+1)\alpha_0}
.
\end{align*}

{\bf Bivariate moments: }
Fix $i, j \in [k]$ with $i \neq j$, and let $\alpha' := \alpha + p \cdot
e_i + q \cdot e_j$ for some positive integers $p$ and $q$.
Then
\begin{align*}
\E[\hidden_i^p\hidden_j^q]
& = \frac{Z(\alpha')}{Z(\alpha)} \\
& = \frac{\Gamma(\alpha_i+p) \cdot \Gamma(\alpha_j+q)}
{\Gamma(\alpha_i) \cdot \Gamma(\alpha_j)}
\cdot \frac{\Gamma(\alpha_0)} {\Gamma(\alpha_0+p+q)}
\\
& = \frac
{\bigl( (\alpha_i+p-1) (\alpha_i+p-2) \dotsb \alpha_i \bigr)
\cdot \bigl( (\alpha_j+q-1) (\alpha_j+q-2) \dotsb \alpha_j \bigr)}
{(\alpha_0+p+q-1) (\alpha_0+p+q-2) \dotsb \alpha_0}
.
\end{align*}
In particular,
\begin{align*}
\E[\hidden_i\hidden_j] & = \frac{\alpha_i\alpha_j}
{(\alpha_0+1)\alpha_0} \\
\E[\hidden_i^2\hidden_j] & = \frac{(\alpha_i+1)\alpha_i\alpha_j}
{(\alpha_0+2)(\alpha_0+1)\alpha_0}
.
\end{align*}

{\bf Trivariate moments: }
{\def\k{{\kappa}}
Fix $i, j, \k \in [k]$ all distinct, and let $\alpha' := \alpha + e_i + e_j
+ e_\k$.
Then
\begin{align*}
\E[\hidden_i\hidden_j\hidden_\k]
& = \frac{Z(\alpha')}{Z(\alpha)} \\
& = \frac{\Gamma(\alpha_i+1) \cdot \Gamma(\alpha_j+1) \cdot
\Gamma(\alpha_\k+1)}
{\Gamma(\alpha_i) \cdot \Gamma(\alpha_j) \cdot \Gamma(\alpha_\k)}
\cdot \frac{\Gamma(\alpha_0)} {\Gamma(\alpha_0+3)}
\\
& = \frac
{\alpha_i \alpha_j \alpha_\k}
{(\alpha_0+2) (\alpha_0+1) \alpha_0}
.
\end{align*}
}

{\bf Completing the proof:} The proof for the second moment matrix and the third moment tensor
follows by observing that each component agrees with the above
expressions. For the final claim,
\begin{multline*}
\E[(\hidden \otimes \hidden) \dotp{v,\hidden}]
= \frac1{(\alpha_0+2)(\alpha_0+1)\alpha_0}
\Bigl( \dotp{v,\alpha} (\alpha \otimes \alpha)
+ \sum_{i=1}^k \alpha_i v_i \bigl( e_i \otimes \alpha \bigr)
+ \sum_{i=1}^k \alpha_i v_i \bigl( \alpha \otimes e_i \bigr)
\\
+ \sum_{i=1}^k \alpha_i \dotp{v,\alpha} (e_i \otimes e_i)
+ 2 \sum_{i=1}^k \alpha_i v_i (e_i \otimes e_i)
\Bigr)
\\
= \frac1{(\alpha_0+2)(\alpha_0+1)\alpha_0}
\Bigl( \dotp{v,\alpha} \alpha \alpha^\t
+ \diag(\alpha) v\alpha^\t
+ \alpha v^\t \diag(\alpha)
\\
+ \dotp{v,\alpha} \diag(\alpha)
+ 2 \diag(v) \diag(\alpha)
\Bigr)
\end{multline*}
which completes the proof.
\end{proof}

\subsection{The proof of Lemma~\ref{lemma:lda_pairs_triples}}

\begin{proof} 
Observe:
\[
\E[x_1] = O\E[\hidden]
\]
and
\[
\E[x_1 x_2^\t ] = \E[\E[x_1 x_2^\t|\hidden] ] =  O \E[\hidden\hidden^\t] O^\t
\]
Define the analogous quantity: 
\[
\Pairs_h  = \E[\hidden \hidden^\t] - 
\frac{\alpha_0}{\alpha_0+1} 
\E[\hidden]\E[\hidden]^\t  
\]
and so:
\[
\Pairs_{\alpha_0}  = O \Pairs_h O^\t
\]
Observe:
\begin{align*}
\Pairs_h 
& = \E[\hidden \hidden^\t] - 
\frac1{(\alpha_0+1)\alpha_0} 
\alpha\alpha^\t \\
& = \frac1{(\alpha_0+1)\alpha_0} \diag(\alpha) 
\end{align*}
Hence,
\[
\Pairs_{\alpha_0}  = O \Pairs_h O^\t = \frac1{(\alpha_0+1)\alpha_0}   O \diag(\alpha)  O^\t 
\]
which proves the first claim.

Also, define:
\begin{multline*}
\Triples_h(v)
: = \E[(\hidden \otimes \hidden) \dotp{v,\hidden}]
-
\frac{\alpha_0}{\alpha_0+2}
\Bigl(\E[\hidden \hidden^\t]  v\E[\hidden]^\t
+ \E[\hidden] v^\t \E[\hidden \hidden^\t] + \dotp{v,\E[\hidden]} \E[\hidden \hidden^\t]
\Bigr) 
\\
+\frac{2\alpha_0^2}{(\alpha_0+2)(\alpha_0+1)}
\dotp{v,\E[\hidden]} \E[\hidden] \E[\hidden]^\t
\end{multline*}

Since
\begin{align*}
\E[x_1 x_2^\t \dotp{\eta,x_3}] 
& = 
\E[ \E[x_1 x_2^\t \dotp{\eta,x_3}|\hidden]]
\\
& = 
O \E[\hidden \hidden^\t \dotp{\eta ,O\hidden}] O^\t
\\
& = 
O \E[\hidden \hidden^\t \dotp{O^\t \eta ,\hidden}] O^\t
\end{align*}
we have
\[
\Triples_{\alpha_0}(\eta) = O \Triples_h(O^\t \eta) O^\t
\]

Let us complete the proof by showing:
\[
\Triples_h(v) := \frac{2}{(\alpha_0+2)(\alpha_0+1)\alpha_0} \diag(v) \diag(\alpha)
\]
Observe:
\begin{multline*}
\frac{2}{(\alpha_0+2)(\alpha_0+1)\alpha_0} \diag(v) \diag(\alpha)
= \E[(\hidden \otimes \hidden) \dotp{v,\hidden}]
- \frac1{(\alpha_0+2)(\alpha_0+1)\alpha_0}
\Bigl( \dotp{v,\alpha} \alpha \alpha^\t
+ \diag(\alpha) v\alpha^\t
\\
+ \alpha v^\t \diag(\alpha) + \dotp{v,\alpha} \diag(\alpha)
\Bigr)
\end{multline*}
Let us handle each term separately. First,
\[
\frac1{(\alpha_0+2)(\alpha_0+1)\alpha_0}
\dotp{v,\alpha} \alpha \alpha^\t 
= 
\frac{\alpha_0^2}{(\alpha_0+2)(\alpha_0+1)}
\dotp{v,\E[\hidden]} \E[\hidden] \E[\hidden]^\t
\]
Also, since:
\[
\frac1{(\alpha_0+1)\alpha_0} \diag(\alpha)  = 
\E[\hidden \hidden^\t] - 
\frac1{(\alpha_0+1)\alpha_0} 
\alpha\alpha^\t 
\]
we have:
\begin{align*}
& \frac1{(\alpha_0+2)(\alpha_0+1)\alpha_0}
\Bigl( \diag(\alpha) v\alpha^\t
+ \alpha v^\t \diag(\alpha) + \dotp{v,\alpha} \diag(\alpha)
\Bigr)
\\
=&  \frac1{\alpha_0+2}
\Bigl(\E[\hidden \hidden^\t]  v\alpha^\t
+ \alpha v^\t \E[\hidden \hidden^\t] + \dotp{v,\alpha} \E[\hidden \hidden^\t]
\Bigr) -  \frac{3}{(\alpha_0+2)(\alpha_0+1)\alpha_0} \dotp{v,\alpha}
\alpha\alpha^\t 
\\
= & \frac{\alpha_0}{\alpha_0+2}
\Bigl(\E[\hidden \hidden^\t]  v\E[\hidden]^\t
+ \E[\hidden] v^\t \E[\hidden \hidden^\t] + \dotp{v,\E[\hidden]} \E[\hidden \hidden^\t]
\Bigr) 
-  
\frac{3\alpha_0^2}{(\alpha_0+2)(\alpha_0+1)}
\dotp{v,\E[\hidden]} \E[\hidden] \E[\hidden]^\t
\end{align*}
Hence,
\begin{multline*}
\frac{2}{(\alpha_0+2)(\alpha_0+1)\alpha_0} \diag(v) \diag(\alpha)
= \E[(\hidden \otimes \hidden) \dotp{v,\hidden}]
\\
-
\frac{\alpha_0}{\alpha_0+2}
\Bigl(\E[\hidden \hidden^\t]  v\E[\hidden]^\t
+ \E[\hidden] v^\t \E[\hidden \hidden^\t] + \dotp{v,\E[\hidden]} \E[\hidden \hidden^\t]
\Bigr) 
\\
+\frac{2\alpha_0^2}{(\alpha_0+2)(\alpha_0+1)}
\dotp{v,\E[\hidden]} \E[\hidden] \E[\hidden]^\t
\end{multline*}
which proves the claim.
\end{proof}

\section{Sample Complexity Analysis}

Throughout, we work in a canonical form. Define:
\[
\canO:= \frac1{\sqrt{(\alpha_0+1)\alpha_0}} O
\diag(\sqrt{\alpha_1},\sqrt{\alpha_2},\dotsc,\sqrt{\alpha_k})
\]
Under this transformation, we have:
\[
\Pairs_{\alpha_0} = \frac1{(\alpha_0+1)\alpha_0} O \diag(\alpha) O^\t
= \canO \canO^\t
\]
Using the definition:
\[
\gamma_i := 2
\sqrt{\frac{\alpha_0(\alpha_0+1)}{(\alpha_0+2)^2} \ \frac{1}{\alpha_i}}
\]
we also have that :
\[
\Triples_{\alpha_0}(\eta) =\canO
\diag(\canO^\t \eta) \diag(\gamma) \canO^\t
\]
Hence, we can consider $\gamma_i$ to be the effective skewness.
Let us also define:
\[
\pmin := \min_i \frac{\alpha_i}{\alpha_0}
\] 
Since $\alpha_i \leq \alpha_0$, we have that:
\[
 \frac{1}{\sqrt{ \alpha_0+2}} \leq \gamma_i \leq 2 \frac1{\sqrt{\pmin(\alpha_0+2)}}
\]
Note that:
\[
\sigma_k(O) \sqrt{\frac{\pmin}{\alpha_0+1}} \leq \sigma_k(\canO) \leq 1
\]
and
\[
\sigma_1(\canO) \leq \sigma_1(O) \frac1{\sqrt{\alpha_0+1}} \leq \frac1{\sqrt{\alpha_0+1}}
\]
where $\sigma_j(\cdot)$ denotes the $j$-th largest singular
value. These lower bounds are relevant for lower bounding certain
singular values in our analysis. 

We use $\|M\|$ to denote the spectral norm of a matrix $M$. Let us
suppose that for all $\eta$,
\begin{eqnarray*}
\| \wh \Pairsa- \Pairsa  \| &= & E_P \\
\| \Triplesa (\eta)  - \wh \Triplesa (\eta) \| 
& \leq & \|\eta\| E_T
\end{eqnarray*}
for some $E_P$ and $E_T$ (which we set later).

\subsection{Perturbation Lemmas}

Let $\wh \Pairsa_{,k}$ be the best rank $k$ approximation to
$\Pairsa$.  We have that $\h W$, as defined in
Algorithm~\ref{alg:LDA_empirical}, whitens $\wh \Pairsa_{,k}$, \emph{i.e.}, \[
\h W^\t \ \wh \Pairsa_{,k} \h W= \I \, .
\]
Let
\[
\h W^\t  \Pairsa \h W = A D A^\t
\]
be an SVD of $\h W^\t  \Pairsa \h W$, where $A \in \R^{k \times
  k}$.  Define:
\[
W: = \h W  A D^{-1/2} A^\t
\]
and observe that $W$ also whitens $\Pairsa$, \emph{i.e.}, \[
W^\t \Pairsa W = (A D^{-1/2} A^\t )^\t \h W^\t  \Pairsa \h W (A D^{-1/2} A^\t )= \I
\]
Due to sampling error, the $\range(W)$ may not equal the $\range(\Pairsa)$. 

Define:
\[
M:=W^\t  \canO , \quad \quad \h M=\h W^\t  \canO
\]

\begin{lemma}\label{lemma:pairs_error}
Let $\Pi_W$ be the orthogonal projection onto the range of $W$ and
$\Pi$ be the orthogonal projection onto the range of $O$.  Suppose
$E_P \leq \sigma_k(\Pairsa) /2$. We have that: 
\begin{eqnarray*}
  \|M\| & = & 1 \\
  \|\h M\| & \leq & 2\\
  \|\h W\| & \leq & \frac{2}{\sigma_k(\canO)}\\
  \\
  \|\h W^+\| & \leq & 2 \sigma_1(\canO) \\
  \\
  \| W^+\| & \leq & 3 \sigma_1(\canO) \\
  \\
 \|M - \h M\|  
  & \leq & 
  \frac{4 }{\sigma_k(\canO)^2} E_P\\
  \\
  \|\h W^+ - W^+\|
  & \leq &  
  \frac{6 \sigma_1(\canO) }{\sigma_k(\canO)^2} E_P\\
  \\
  \| \Pi -  \Pi_W\| & \leq & \frac{4}{\sigma_k(\canO)^2} E_P
\end{eqnarray*}
\end{lemma}

\begin{proof}
Since $W$ whitens $\Pairsa$,  we have $M M^\t = W^\t \canO \canO^\t W=
\I$ and 
\[
\|M \| =1
\]
By Weyl's theorem (see Lemma~\ref{lemma:weyl}),
\[
\|\h W\|^2 = \frac{1}{\sigma_k(\wh\Pairsa)} \leq
\frac{1}{\sigma_k(\Pairsa)-\| \wh \Pairsa- \Pairsa  \|} 
\leq  \frac{2}{\sigma_k(\Pairsa)}
=  \frac{2}{\sigma_k(\canO)^2}
\]

Also, $\h W =  W  A D^{1/2} A^\t$ so that 
$\h M = A D^{1/2} A^\t M$ and
\begin{eqnarray*}
\| M - \h M\| 
& = &
\|M  -  A D^{1/2} A^\t M \| \\
& \leq &
\| M\| \|\I -A D^{1/2} A^\t \| \\
& = &
\| \I - D^{1/2} \| \\
& \leq &
\| \I - D^{1/2} \| \|\I +D^{1/2} \|\\
& = &
\|\I -D \| \\
\end{eqnarray*}
where we have used that $D\succeq 0$ and $D$ is diagonal.

We can bound this as follows:
\begin{eqnarray*}
\|\I -D \|  = &=&
\| \I - A D A^\t \| \\
&=&
\| \I - \h W^\t  \Pairsa \h W \| \\
&=&
 \|\h W^\t ( \wh \Pairsa_{,k} - \Pairsa )\h W \|\\
&\leq&
\|\h W\|^2 \| \wh \Pairsa_{,k}- \Pairsa  \|\\
&\leq&
\|\h W\|^2 ( \|\wh \Pairsa_{,k} -\wh \Pairsa\| + \| \wh \Pairsa- \Pairsa  \| )\\
&=&
\|\h W\|^2 ( \sigma_{k+1}(\wh \Pairsa)+\| \wh \Pairsa- \Pairsa  \| )\\
&\leq&
2\|\h W\|^2 \| \wh \Pairsa- \Pairsa  \|\\
&\leq&
4 \frac1{\sigma_k(\canO)^2} E_P
\end{eqnarray*}
using Weyl's theorem in the second to last
step. 

This implies $\|\I -D \| \leq 4 \frac1{\sigma_k(\canO)^2} E_P \leq 2 $ and  so
$\|D\|\leq 3$. 
Since $\h M = A D^{1/2} A^\t M$,
\[
\|\h M\|^2 \leq \| M\|^2 \|D\| \leq 3
\, .
\]

Again, by Weyl's theorem,
\[
\|\h W^+\|^2 = \sigma_1(\wh\Pairsa) \leq \sigma_1(\wh\Pairsa)+E_P \leq 1.5 \sigma_1(\Pairsa) = 1.5 \sigma_1(\canO)^2
\]
Using that $W = \h W  A D^{-1/2} A^\t$, we have:
\[
\| W^+\|^2 \leq \|\h W^+\|^2 \|D\| \leq 4.5 \sigma_1(\canO)^2
\]
and
\[
\|\h W^+ - W^+\| \leq  \|\h W^+\| \|\I -D^{1/2} \| \leq  \|\h W^+\|
\|\I -D \|  
\leq  \frac{6 \sigma_1(\canO) }{\sigma_k(\canO)^2} E_P
\]
which completes the argument for the first set of claims.

We now prove the final claim.  Let
$\Theta$ be the matrix of canonical angles between $\range(\Pairsa)$
and $\range(\wh \Pairsa_{,k})$. 
By  Wedin's theorem (see Lemma~\ref{lemma:wedinB}) (and noting that
the $k$-th singular value of $\wh \Pairsa_{,k}$ is greater than
$\sigma_k(\Pairsa)/2$), we have
\[
\|\sin \Theta\| 
\leq 
2 \frac{\|\Pairsa -\wh \Pairsa_{,k}\|}{\sigma_k(\Pairsa)} 
\leq 
2 \frac{\|\wh \Pairsa_{,k} -\wh \Pairsa\| + \| \wh \Pairsa- \Pairsa  \|}{\sigma_k(\Pairsa)} 
\leq 
4 \frac{E_P}{\sigma_k(\Pairsa)}
\]
Using Lemma~\ref{lemma:sin_project}, $\| \Pi - \Pi_W \|= 
\|\sin \Theta\| $,  which completes the proof.
\end{proof}

\begin{lemma}\label{lemma:triples_error}
Suppose $ E_P \leq \sigma_k(\Pairsa) /2$. For $\|\theta\|=1$, we have:
\[
\| W^\t \Triplesa (W \theta) W - \h W^\t \wh \Triplesa (\h W \theta)
\h W\| 
\ \leq \
c \left( \frac{E_P }{\sqrt{\pmin(\alpha_0+2)} \ \sigma_k(\canO)^{2}} 
  +\frac{E_T}{\sigma_k(\canO)^{3}} \right)
\]
where $c$ is a universal constant.
\end{lemma}

\begin{proof}
We have:
\begin{multline*}
\| W^\t \Triplesa (W \theta) W - \h W^\t \wh \Triplesa (\h W \theta)
\h W\|
\leq
\| W^\t \Triplesa (W \theta) W - \h W^\t \Triplesa (\h W \theta) \h W \|\\
+
\|\h W^\t \Triplesa (\h W \theta) \h W - \h W^\t \wh \Triplesa (\h W
\theta) \h W\|
\end{multline*}
For the second term:
\begin{eqnarray*}
\|\h W^\t \Triplesa (\h W \theta) \h W - \h W^\t \wh \Triplesa (\h W
\theta) \h W\|
& \leq &
\| \h W \|^2 \| \Triplesa (\h W \theta)  - \wh \Triplesa (\h W \theta) \|\\
& \leq &
\|\h W \|^3 E_T\\
& \leq &
\frac{8 }{\sigma_k(\canO)^{3}} E_T
\end{eqnarray*}

For the first term, by expanding out the terms and using the bounds in
Lemma~\ref{lemma:pairs_error}, we have:
\begin{eqnarray}
&& \| W^\t \Triplesa (W \theta) W - \h W^\t \Triplesa (\h W \theta) \h
W \| \nonumber \\
& = &
\|M \diag( M^\t \theta) \diag(\gamma)  M^\t
-
\h M \diag( \h M^\t \theta) \diag(\gamma)  \h M^\t
\|\nonumber \\
& \leq &
\| M \diag( M^\t \theta) \diag(\gamma)  M^\t
-
\h M \diag( M^\t \theta) \diag(\gamma)  \h M^\t
\|\nonumber \\
&&
+ \|\h M \diag(  (M-\h M) ^\t \theta) \diag(\gamma)  \h M^\t
\| \nonumber \\
& \leq &
\| M \diag( M^\t \theta) \diag(\gamma)  M^\t
-
\h M \diag( M^\t \theta) \diag(\gamma)  \h M^\t
\| + \max_i \gamma_i \|\h M\|^2 \| M-\h M\| \nonumber \\
& \leq &
c \max_i \gamma_i \| M-\h M\| \label{eq:simalar_arg}
\end{eqnarray}
for some constant $c$ (where the last step follows from expanding out terms).
\end{proof}

\subsection{SVD Accuracy}

Let $\sigma_i$ and $v_i$ denote the corresponding $i$-th singular
value (in increasing order) and vector of $W^\t \Triplesa (W \theta)
W$. Similarly, let $\hat v_i$ and $\hat \sigma_i$ denote the
corresponding $i$-th singular value (in increasing order) and vector
of $\h W^\t \wh \Triplesa (\h W \theta) \h W $.  For
convenience, choose the sign of $\hat v_i$ so that $\dotp{ v_i ,\hat
  v_i}\geq 0$.

The following lemma characterizes the accuracy of the SVD:

\begin{lemma}[SVD Accuracy] \label{lemma:E}
Suppose $ E_P \leq \sigma_k(\Pairsa) /2$. With probability greater
than $1-\delta'$, we have for all $i$: 
\[
\|v_i - \hat v_i \| \leq c \frac{k^3 \sqrt{\alpha_0+2} }{\delta'} 
\left( \frac{E_P }{\sqrt{\pmin(\alpha_0+2)} \ \sigma_k(\canO)^{2}} 
  +\frac{E_T}{\sigma_k(\canO)^{3}} \right)
\]
for some universal constant $c$.
\end{lemma}

First, let us provide a few lemmas. Let 
\[
\| W^\t \Triplesa (W \theta) W - 
\h W^\t \wh \Triplesa (\h W \theta) \h W\| 
\leq E
\]
where the bound on $E$ is provided in Lemma~\ref{lemma:triples_error}.

\begin{lemma}
Suppose for all $i$
\begin{align*}
\sigma_i &\geq \Delta\\
|\sigma_i -\sigma_{i+1}| & \geq \Delta
\end{align*}
For all $i$,  $v_i$ and $\hat v_i$, we have:
\[
\|v_i - \hat v_i \| \leq 2 \frac{\sqrt{k} E}{\Delta-E}
\]
where the sign of $\hat v_i$ chosen so $\dotp{\hat v_i ,\hat v_i}
\geq 0$.
\end{lemma}

\begin{proof}
Let $\cos(\theta) = \dotp{v_i, \hat v_i}$ (which is positive since we
assume  $\dotp{ v_i ,\hat v_i}\geq 0$). We have:
\[
\|v_i - \hat v_i \|^2 =2(1-\cos(\theta))  = 
\leq
2(1-\cos^2(\theta)) = 2\sin^2(\theta)
\]
By Weyl's theorem (see Lemma~\ref{lemma:weyl}) and by assumption,
\[
\min_{i} |\hat\sigma_i - \sigma_j| \geq \Delta - E
\]
and
\[
\min_{j\neq i} |\hat\sigma_i - \sigma_j| \geq \min_{j\neq i}
|\sigma_i - \sigma_j| - E \geq \Delta - E
\]
By Wedin's theorem (see Lemma~\ref{lemma:wedin} applied to the split
where $v_i$ and $ v_i$ correspond to the  subspaces
$U_1$ and $ U_1$),
\[
|\sin(\theta)| \leq \sqrt{2} \frac{E_F}{\Delta-E} \leq \sqrt{2} \frac{\sqrt{k} E}{\Delta-E}
\]
\end{proof}

\begin{lemma} 
Fix any $\delta \in (0,1)$ and matrix $A \in \R^{k \times k}$.
Let $\theta \in \R^k$ be a random vector distributed uniformly over
$\sphere^{k-1}$. With probability greater than $1-\delta$, we have
\[
\min_{i \neq j} |\dotp{\theta,A(e_i -
e_j)}| > \frac{\min_{i \neq j} \|A(e_i-e_j)\| \cdot
\delta}{\sqrt{e} k^{2.5}} 
\]
and
\[
\min_{i } |\dotp{\theta,Ae_i }| > \frac{\min_{i} \|Ae_i\| \cdot
\delta}{\sqrt{e} k^{2.5}} 
\]
\end{lemma}

\begin{proof}
By Lemma~\ref{lemma:sanjoy}, for any fixed pair $\{i,j\} \subseteq
[k]$ and $\beta := \delta_0 / \sqrt{e}$,
\[ \Pr\biggl[ |\dotp{\theta,A(e_i-e_j)}| \leq \|A(e_i-e_j)\| \cdot
\frac{1}{\sqrt{k}} \cdot \frac{\delta_0}{\sqrt{e}} \biggr]
\leq exp\left(\frac12 (1 - (\delta_0^2 / e) + \ln (\delta_0^2/e)) \right)
\leq \delta_0
.
\]
Similarly,  for each $i$
\[ \Pr\biggl[ |\dotp{\theta,Ae_i)}| \leq \|Ae_i\| \cdot
\frac{1}{\sqrt{k}} \cdot \frac{\delta_0}{\sqrt{e}} \biggr]
\leq exp\left(\frac12 (1 - (\delta_0^2 / e) + \ln (\delta_0^2/e)) \right)
\leq \delta_0
.
\]
Let $\delta_0 := \delta / k^2$. The claim follows by a union bound over all ${k \choose 2}+k\leq k^2$
possibilities.  
\end{proof}

We now complete the argument.

\begin{proof}[Proof of Lemma~\ref{lemma:E}] Choose
  $A=\diag(\gamma_1,\gamma_2,\dotsc,\gamma_k) M^\t $, where $M=W^\t  \canO$. The proof of
  Theorem~\ref{thm:skew} shows $\dotp{e_i, A\theta}$ are the singular
  values. Also the minimal singular value of $A$ is greater than
  $ \min_i \gamma_i \geq \frac{1}{\sqrt{ \alpha_0+2}} \leq$ (since $MM^\t=\I$). Hence, we have:
\begin{eqnarray*}
\sigma_i &\geq \frac{\delta}{2 k^{2.5}\sqrt{ \alpha_0+2}} &:=\Delta \\
|\sigma_i -\sigma_{i+1}| & \geq 
\frac{\delta}{2 k^{2.5}\sqrt{ \alpha_0+2}} 
\end{eqnarray*}
Suppose $E \leq \Delta/2$. Here,
\[
\|v_i - \hat v_i \| \leq 2 \frac{\sqrt{k} E}{\Delta-E} \leq 4 \frac{\sqrt{k} E}{\Delta}
=
8 \frac{k^3 \sqrt{ \alpha_0+2} }{\delta}D
\]
Also, since $\|v_i - \hat v_i \| \leq 2$ the above also holds for
$E > \Delta/2$, which proves the first claim.
\end{proof}

\subsection{Reconstruction Accuracy}

\begin{lemma}\label{lemma:reconstruct}
Suppose $ E_P \leq \sigma_k(\Pairsa) /2$. With probability greater
than $1-\delta'$, we have that for all $i$: 
\[
\|O_i - \frac{1}{\h Z_i} \ (\h W^+)^\t  \hat v_i \| 
\leq  c
\frac{ (\alpha_0+2)^2 k^3}{\pmin^{2} \sigma_k(O)^{3} \delta' }  
\max\{E_P,E_T\}
\]
where $\{O_1,O_2,\ldots O_k\}$ is some permutation of the columns of $O$.
\end{lemma}

\begin{proof}
First, observe that $W$ whitens $\Pi_W \Pairsa \Pi_W^\t$. To see,
observe that $\Pi_W^\t = \Pi_W$ (since $\Pi_W$ is an orthogonal 
projection) and  $\Pi_W^\t W = \Pi_W W = W$; so
\[
W^\t \left( \Pi_W \Pairsa \Pi_W^\t\right) W
= 
(\Pi_W^\t W)^\t \Pairsa (\Pi_W^\t W) 
= 
W^\t \Pairsa W = \I
\]
Using the definition $M=W^\t  \canO = W^\t  \Pi_W \canO $, we have:
\[
W^\t \Triplesa(W\theta) W =  M \diag( M^\t \theta) \diag(\gamma_1,\gamma_2,\dotsc,\gamma_k)  M^\t
\]
Since $\range(W) = \range(\Pi_W \canO)$ the proof of Theorem~\ref{thm:skew} shows that:
\[
\Pi_W \canO_i = (W^+)^\t v_i
\]
Define:
\[
Z_i := \frac{2} {(\alpha_0+2) ( W v_i )^\t  \Triplesa( W  v_i) W v_i}
\, ,
\]
Since $v_i=M^\t e_i$ are the singular vectors of $W^\t \Triplesa(W\theta) W$,
we have 
\[
( W v_i )^\t  \Triplesa( W  v_i) W v_i = \gamma_i
\]
and so:
\[
Z_i = \sqrt{\frac{\alpha_i}{(\alpha_0+1)\alpha_0}} \, .
\]
This implies $\canO_i = Z_i O_i$ and so
\[
\Pi_W O_i = \frac{1}{Z_i} \Pi_W \canO_i = \frac1{Z_i} (W^+)^\t v_i
\]
since $\Pi_W \canO_i = (W^+)^\t v_i$. 

Now let us bound the reconstruction error as follows:
\begin{eqnarray*}
& & \|O_i - \frac{1}{\h Z_i} \ (\h W^+)^\t  \hat v_i \| \\
& \leq & \|O_i - \Pi_W O_i \| 
+\|\Pi_W O_i - \frac{1}{\h Z_i} \ (\h W^+)^\t  \hat v_i \|  \\
&= & \|\Pi O_i- \Pi_W O_i \| 
+\|\frac1{Z_i} (W^+)^\t v_i - \frac{1}{\h Z_i} \ (\h W^+)^\t
\hat v_i \|  \\
&\leq & \|\Pi - \Pi_W \| \|O_i\|+
\|\frac1{Z_i} (W^+)^\t v_i - \frac1{Z_i} (W^+)^\t  \hat v_i \|  
+\|\frac1{Z_i} (W^+)^\t \hat v_i - \frac{1}{\h Z_i} \ (\h W^+)^\t
\hat v_i \|  \\
&\leq & \|\Pi - \Pi_W \| +
\frac{\|W^+\|}{Z_i}  \| v_i -  \hat v_i \|  
+\|\frac1{Z_i} W^+ - \frac{1}{\h Z_i} \ \h W^+  \|  \\
&\leq & \|\Pi - \Pi_W \| +
\frac{\|W^+\|}{Z_i}  \| v_i -  \hat v_i \|  
+ \|\frac1{Z_i} W^+ - \frac1{Z_i} \h W^+  \| +
\|\frac1{ Z_i} \h W^+ - \frac{1}{\h Z_i} \ \h W^+  \| \\
&\leq & \|\Pi - \Pi_W \| +
\frac{\|W^+\|}{Z_i}  \| v_i -  \hat v_i \|  
+\frac1{Z_i} \left\| W^+ - \ \h W^+\right\|+
\|\h W^+\|  \left|\frac1{Z_i} - \frac1{\h Z_i}  \right| 
\end{eqnarray*}

For bounding $|\frac1{Z_i} - \frac1{\h Z_i}  |$, first observe:
\begin{eqnarray*}
&&|(W v_i )^\t \Triplesa( W v_i)  W  v_i -
(\h W  \h v_i )^\t \wh\Triplesa(\h W \h v_i) \h W \h v_i 
|\\
&\leq &
|(W v_i )^\t \Triplesa( W v_i)  W  v_i -
(W \h v_i )^\t \Triplesa(W \h v_i) W \h v_i 
|\\
&&
+
|(W \h v_i )^\t \Triplesa(W \h v_i) W \h v_i  -
(\h W  \h v_i )^\t 
\wh\Triplesa(\h W \h v_i) \h W \hat v_i 
|
\\
&\leq &
c \|v_i -\h v_i\| \max_i \gamma_i
+
\| W^\t \Triplesa (W \h v_i) W - \h W^\t \wh \Triplesa (\h W \h v_i)
\h W\| \\
\end{eqnarray*}
where $c$ is a constant and where that last step uses an argument
similar to that of Equation~\ref{eq:simalar_arg} (along with the
bounds $\|W^\t \canO\|=1$, $\|M^\t v_i\|\leq 1$ and $\|M^\t \h
v_i\|\leq 1$). Continuing,
\begin{eqnarray*}
&&|(W v_i )^\t \Triplesa( W v_i)  W  v_i -
(\h W  \h v_i )^\t \wh\Triplesa(\h W \h v_i) \h W \h v_i 
|\\
&\leq &
c_1 \|v_i -\h v_i\| \max_i \gamma_i
+
c_2 \left( \frac{E_P }{\sqrt{\pmin(\alpha_0+2)} \ \sigma_k(\canO)^{2}} 
  +\frac{E_T}{\sigma_k(\canO)^{3}} \right) \\
&\leq &
c_3 \frac{k^3 }{\delta' \sqrt{\pmin}} 
\left( \frac{E_P }{\sqrt{\pmin(\alpha_0+2)} \ \sigma_k(\canO)^{2}} 
  +\frac{E_T}{\sigma_k(\canO)^{3}} \right)
\end{eqnarray*}
using that $\gamma_i \leq 2 \frac1{\sqrt{\pmin(\alpha_0+2)}}$ in the
last step (for constants $c_1,c_2,c_3$).

The fourth
term is bounded as follows:
\begin{eqnarray*}
&&\|\h W^+\|  \left|\frac1{Z_i} - \frac1{\h Z_i}  \right|  \\
&= & \|\h
W^+\| \frac{\alpha_0+2}{2}\ 
|(W v_i )^\t \Triplesa( W v_i)  W  v_i -
(\h W  \h v_i )^\t \wh\Triplesa(\h W \h v_i) \h W \h v_i |
\\
& \leq & c_1 \|\h W^+\|  \frac{k^3 (\alpha_0+2) }{\delta' \sqrt{\pmin}} 
\left( \frac{E_P }{\sqrt{\pmin(\alpha_0+2)} \ \sigma_k(\canO)^{2}} 
  +\frac{E_T}{\sigma_k(\canO)^{3}} \right)\\
& \leq & c_2 \sigma_1(\canO)  \frac{k^3 (\alpha_0+2) }{\delta' \sqrt{\pmin}} 
\left( \frac{E_P }{\sqrt{\pmin(\alpha_0+2)} \ \sigma_k(\canO)^{2}} 
  +\frac{E_T}{\sigma_k(\canO)^{3}} \right)\\
\end{eqnarray*}
for constants $c_1,c_2,c_3$.

We have:
\[
\frac{\|W^+\|}{Z_i} \leq c \frac{\sigma_1(\canO)}{Z_i} = c \sigma_1(\canO)
\sqrt{\frac{\alpha_0 (\alpha_0+1) }{\alpha_i}} \leq c \sigma_1(\canO) \sqrt{\frac{\alpha_0+1}{\pmin}}
\]
(for a constant $c$), so for the second term:
\[
\frac{\|W^+\|}{Z_i}  \| v_i -  \hat v_i \|   \leq  
c \sigma_1(\canO)  \frac{k^3 (\alpha_0+2) }{ \delta' \sqrt{\pmin}} 
\left( \frac{E_P }{\sqrt{\pmin(\alpha_0+2)} \ \sigma_k(\canO)^{2}} 
  +\frac{E_T}{\sigma_k(\canO)^{3}} \right)\\
\]
The remaining terms can be show to be of lower order, so that:
\begin{eqnarray*}
\|O_i - \frac{1}{\h Z_i} \ (\h W^+)^\t  \hat v_i \| 
& \leq & c \sigma_1(\canO)  \frac{k^3 (\alpha_0+1) }{\delta' \sqrt{\pmin}} 
\left( \frac{E_P }{\sqrt{\pmin(\alpha_0+2)} \ \sigma_k(\canO)^{2}} 
  +\frac{E_T}{\sigma_k(\canO)^{3}} \right)\\
\\
& \leq & c_2 \frac{k^3 \sqrt{\alpha_0+2} }{\delta' \sqrt{\pmin}} 
\left( \frac{(\alpha_0+2)^{1/2}E_P }{\pmin^{3/2} \ \sigma_k(O)^{2}} 
  +\frac{(\alpha_0+2)^{3/2}E_T}{\pmin^{3/2} \sigma_k(O)^{3}} \right)\\
\\
&= & c_2 \frac{k^3 (\alpha_0+2) }{\pmin^{2} \delta'} 
\left( \frac{E_P }{ \ \sigma_k(O)^{2}} 
  +\frac{(\alpha_0+2)E_T}{\sigma_k(O)^{3}} \right)
\end{eqnarray*}
using that $
\sigma_k(\canO)  \geq \sigma_k(O) \sqrt{\frac{\pmin}{\alpha_0+1}} 
$ and
$\sigma_1(\canO) \leq \frac1{\sqrt{\alpha_0+1}}  $.
\end{proof}

\subsection{Completing the proof}

\begin{proof}[Proof of Theorem~\ref{thm:samples}]
Lemma~\ref{lemma:topic-concentration} and the definition of $\Pairsa$
and $\Triplesa$ imply that:
\begin{eqnarray*}
\| \wh \Pairsa- \Pairsa  \| &\leq & 3 \frac{1 +
\sqrt{\ln(3/\delta)}}{\sqrt{N}} \\
\| \Triplesa (\eta)  - \wh \Triplesa (\eta) \| 
& \leq & c \frac{\|\eta\|_2 ( 1 +
\sqrt{\ln(3/\delta)})}{\sqrt{N}}
\end{eqnarray*}
for a constant $c$ (by expanding out the terms and by using $\delta/3$
results in a total error probability of $\delta$). Hence, we can take $E_P=E_T=c \frac{1 +
  \sqrt{\ln(1/\delta)}}{\sqrt{N}}$. Since
 $N\geq \left( \frac{(\alpha_0+1 )(6 + 6\sqrt{\ln(3/\delta)})}{\pmin
     \sigma_k(O)^2}\right)^2 \geq \left( \frac{6 + 6\sqrt{\ln(3/\delta)}}{
     \sigma_k(\Pairsa)}\right)^2$, the condition $ E_P \leq
 \sigma_k(\Pairsa) /2$ is 
satisfied. The proof is completed using Lemma~\ref{lemma:reconstruct}.
\end{proof}

\section{Tail Inequalities}

\begin{lemma}[Lemma A.1 in  \cite{AHKparams}]
\label{lemma:topic-concentration}
Fix $\delta \in (0,1)$.
Let $x_1, x_2, x_3$ are random variables in which
$\|x_1\|,\|x_2\|,\|x_3\|$ are bounded by $1$, almost surely.
Let $\h E[x_1]$ be the empirical average of $N$ independent copies of
$x_1$;  let $\h E[x_1 x_2^\t]$ be the empirical average of $N$ independent copies of
$x_1 x_2^\t$; let $\h E[x_1 x_2^\t \dotp{\eta,x_3}]$. be the empirical average of $N$ independent copies of
$x_1 x_2^\t \dotp{\eta,x_3}$.
Then
\begin{enumerate}
\item $\Pr\left[ \|\h E[x_1] - E[x_1 ]\|_\F \leq \frac{1 +
\sqrt{\ln(1/\delta)}}{\sqrt{N}} \right] \geq 1-\delta$
\item $\Pr\left[ \|\h E[x_1 x_2^\t] - E[x_1 x_2^\t]\|_\F \leq \frac{1 +
\sqrt{\ln(1/\delta)}}{\sqrt{N}} \right] \geq 1-\delta$
\item $\Pr\left[ \forall \eta\in\R^d,\
\|\h E[x_1 x_2^\t \dotp{\eta,x_3}] - E[x_1 x_2^\t \dotp{\eta,x_3}]\|_\F \leq \frac{\|\eta\|_2 ( 1 +
\sqrt{\ln(1/\delta)})}{\sqrt{N}} \right] \geq 1-\delta$.
\end{enumerate}
\end{lemma}

\begin{lemma}[\citet{DG03}]
\label{lemma:sanjoy}
Let $\theta \in \R^n$ be a random vector distributed uniformly over
$\sphere^{n-1}$, and fix a vector $v \in \R^n$.
\begin{enumerate}
\item If $\beta \in (0,1)$, then
\[
\Pr\biggl[ |\dotp{\theta,v}| \leq \|v\| \cdot \frac{1}{\sqrt{n}}
\cdot \beta \biggr] \leq \exp\biggl(\frac12(1 - \beta^2 + \ln
\beta^2)\biggr)
.
\]

\item If $\beta > 1$, then
\[
\Pr\biggl[ |\dotp{\theta,v}| \geq \|v\| \cdot \frac{1}{\sqrt{n}}
\cdot \beta \biggr] \leq \exp\biggl(\frac12(1 - \beta^2 + \ln
\beta^2)\biggr)
.
\]

\end{enumerate}
\end{lemma}
\begin{proof}
This is a special case of Lemma 2.2 from~\citet{DG03}.
\end{proof}

\section{Matrix Perturbation Lemmas}

\begin{lemma}[Weyl's theorem; Theorem 4.11, p.~204 in~\citet{SS90}] \label{lemma:weyl}
Let $A, E\in \R^{m \times n}$ with $m \geq n$ be given.
Then
\[ \max_{i \in [n]} |\sigma_i(A+E) - \sigma_i(A)| \leq \|E\|
. \]
\end{lemma}

\begin{lemma}[Wedin's theorem; Theorem 4.1, p.~260 in~\citet{SS90})] \label{lemma:wedin}
Let $A, E \in \R^{m \times n}$ with $m \geq n$ be given.
Let $A$ have the singular value decomposition
\[
\left[ \begin{array}{c} U_1^\top \\ U_2^\top \\ U_3^\top \end{array} \right]
A \left[ \begin{array}{cc} V_1 & V_2 \end{array} \right]
=
\left[ \begin{array}{cc} \Sigma_1 & 0 \\ 0 & \Sigma_2 \\ 0 & 0 \end{array}
\right]
.
\]
Here, we do not suppose $\Sigma_1$ and $\Sigma_2$ have singular
values in any order. Let $\tl A := A + E$, with analogous singular value decomposition
$(\tl U_1, \tl U_2, \tl U_3, \tl \Sigma_1, \tl \Sigma_2, \tl V_1 \tl
V_2)$ (again with no ordering to the singular values).
Let $\Phi$ be the matrix of canonical angles between $\range(U_1)$ and
$\range(\tl U_1)$, and $\Theta$ be the matrix of canonical angles between
$\range(V_1)$ and $\range(\tl V_1)$. Suppose there exists a $\delta$
such that:
\[
\min_{i,j} |[ \Sigma_1]_{i,i} - [\Sigma_2]_{j,j}|>\delta  \,
\textrm{ and }
\min_{i,i} |[ \Sigma_1]_{i,i} |>\delta , 
\]
then
$$
\|\sin \Phi\|_F^2 + \|\sin \Theta\|_F^2  \leq \frac{2\|E\|_F^2}{\delta^2}.
$$
\end{lemma}

\begin{lemma}[Wedin's theorem; Theorem 4.4, p.~262 in~\citet{SS90}.] \label{lemma:wedinB}
Let $A, E \in \R^{m \times n}$ with $m \geq n$ be given.
Let $A$ have the singular value decomposition
\[
\left[ \begin{array}{c} U_1^\top \\ U_2^\top \\ U_3^\top \end{array} \right]
A \left[ \begin{array}{cc} V_1 & V_2 \end{array} \right]
=
\left[ \begin{array}{cc} \Sigma_1 & 0 \\ 0 & \Sigma_2 \\ 0 & 0 \end{array}
\right]
.
\]
Let $\tl A := A + E$, with analogous singular value decomposition
$(\tl U_1, \tl U_2, \tl U_3, \tl \Sigma_1, \tl \Sigma_2, \tl V_1 \tl V_2)$.
Let $\Phi$ be the matrix of canonical angles between $\range(U_1)$ and
$\range(\tl U_1)$, and $\Theta$ be the matrix of canonical angles between
$\range(V_1)$ and $\range(\tl V_1)$.
If there exists $\delta, \alpha > 0$ such that
$\min_i \sigma_i(\tl \Sigma_1) \geq \alpha + \delta$ and
$\max_i \sigma_i(\Sigma_2) \leq \alpha$, then
$$
\max\{\|\sin \Phi\|_2, \|\sin \Theta\|_2\} \leq \frac{\|E\|_2}{\delta}.
$$
\end{lemma}

\begin{lemma} \label{lemma:sin_project}
Let $\Theta$ be the matrix of canonical angles between $\range(X)$ and
$\range(Y)$. Let $\Pi_X$ and $\Pi_Y$ be the orthogonal projections
onto $\range(X)$ and
$\range(Y)$, respectively. We have:
\[
\|\Pi_X - \Pi_Y \| = \|\sin \Theta\|
\]
\end{lemma}
\begin{proof}
See Theorem 4.5, p.~92, and Corollary 4.6, p.~93, in~\citet{SS90}.
\end{proof}

\section{Illustrative empirical results}
\label{appendix:experiments}

We applied Algorithm \ref{alg:LDA_empirical} to the UCI ``Bag of Words''
dataset comprised of New York Times articles.
This data set has $300000$ articles and a vocabular of size $d = 102660$;
we set $k = 50$ and $\alpha_0 = 0$.
Instead of using a single random $\theta$ and obtaining singular vectors of
$\h{W}^\t \Triplesa(\h{W}\theta) \h{W}$, we used the following power
iteration to obtain the singular vectors $\{ \h{v}_1, \h{v}_2, \dotsc,
\h{v}_k \}$:
\begin{center}
{\framebox[0.55\textwidth]{\small\begin{minipage}{0.5\textwidth}
$\{ \h{v}_1, \h{v}_2, \dotsc, \h{v}_k \} \gets$ random
orthonormal basis for $\R^k$.

Repeat:
\begin{enumerate}
\item For $i = 1,2,\dotsc,k$:
\[
\h{v}_i \gets \h{W}^\t \Triplesa(\h{W} \h{v}_i) \h{W} \h{v}_i .
\vspace{-0.3cm}
\]

\item Orthonormalize $\{ \h{v}_1, \h{v}_2, \dotsc, \h{v}_k \}$.
\end{enumerate}
\end{minipage}}}
\end{center}
The top $25$ words (ordered by estimated conditional probability value)
from each topic are shown below.

\begin{landscape}
\begin{center}
\tiny
\begin{tabular}{|c|c|c|c|c|c|c|c|c|}
\hline
zzz\_held & premature & las & sales & million & com & run & school & women \\
send & guard & como & economic & shares & question & inning & student & team \\
advisory & zzz\_held & los & consumer & public & information & hit & teacher & woman \\
publication & released & zzz\_latin\_trade & major & offering & zzz\_eastern & game & program & job \\
released & publication & articulo & home & source & sport & season & official & sport \\
guard & advisory & telefono & indicator & initial & daily & home & public & cancer \\
zzz\_attn\_editor & send & transmiten & weekly & debt & commentary & right & children & look \\
undatelined & undatelined & fax & order & bond & business & games & high & company \\
night & zzz\_washington\_datelined & una & claim & billion & newspaper & zzz\_dodger & education & group \\
advance & zzz\_istanbul & del & scheduled & share & separate & left & district & percent \\
zzz\_andrew\_pollack & zzz\_attn\_editor & articulos & listed & quarter & spot & team & parent & girl \\
zzz\_douglas\_frantz & zzz\_seth\_mydan & espanol & dates & revenue & marked & start & college & study \\
billion & nyt & paises & jobless & market & today & yankees & money & game \\
zzz\_jennifer & zzz\_johannesburg & sobre & prices & zzz\_calif & zzz\_tom\_oder & pitcher & test & games \\
zzz\_dirk\_johnson & zzz\_afghanistan & financial & price & school & holiday & ball & percent & female \\
zzz\_leslie & zzz\_jane\_perlez & zzz\_america\_latina & market & zzz\_new\_york & need & pitch & system & american \\
cell & zzz\_john\_broder & notas & leading & cash & staffed & manager & kid & number \\
zzz\_linda & zzz\_warren & prohibitivo & retailer & stock & development & lead & federal & season \\
games & zzz\_melbourne & con & economy & percent & toder & night & law & breast \\
zzz\_lee & zzz\_lexington & revista & index & securities & client & homer & need & play \\
zzz\_james\_brooke & zzz\_erik\_eckholm & tiene & retail & zzz\_credit\_suisse\_first\_boston & eta & field & help & zzz\_taliban \\
zzz\_winnipeg & zzz\_bernard\_simon & economia & spending & deal & directed & play & class & right \\
deal & substitute & costo & product & contract & additional & ranger & group & part \\
husband & close & otros & cost & president & reach & win & plan & male \\
zzz\_usc & point & zzz\_paris & producer & expected & washington & hitter & black & high \\
\hline
\end{tabular}

\begin{tabular}{|c|c|c|c|c|c|c|c|c|}
\hline
drug & player & article & palestinian & tax & cup & point & yard & percent \\
patient & zzz\_tiger\_wood & zzz\_new\_york & zzz\_israel & cut & minutes & game & game & stock \\
million & won & misstated & zzz\_israeli & percent & oil & team & play & market \\
company & shot & zzz\_boston\_globe & zzz\_yasser\_arafat & zzz\_bush & water & shot & season & fund \\
doctor & play & zzz\_united\_states & peace & billion & add & play & team & investor \\
companies & round & company & israeli & plan & tablespoon & zzz\_laker & touchdown & companies \\
percent & win & president & israelis & bill & food & season & quarterback & analyst \\
cost & tournament & campaign & leader & taxes & teaspoon & half & coach & money \\
program & tour & zzz\_clinton & official & million & pepper & lead & defense & investment \\
health & right & surname & attack & zzz\_congress & sugar & games & quarter & economy \\
care & par & player & zzz\_bush & zzz\_george\_bush & large & quarter & ball & point \\
billion & final & incorrectly & zzz\_west\_bank & economy & fat & minutes & field & company \\
plan & playing & point & zzz\_palestinian & money & butter & night & pass & quarter \\
medical & major & film & violence & income & sauce & left & run & price \\
treatment & ball & director & security & government & serving & goal & offense & billion \\
zzz\_aid & hit & office & killed & spending & hour & king & line & earning \\
disease & lead & school & talk & federal & fresh & final & running & prices \\
cancer & golf & home & military & pay & pan & played & defensive & firm \\
hospital & guy & misspelled & jewish & republican & taste & scored & zzz\_nfl & index \\
prescription & hole & died & zzz\_jerusalem & zzz\_white\_house & bowl & zzz\_kobe\_bryant & football & growth \\
federal & course & information & soldier & zzz\_senate & cream & rebound & receiver & zzz\_nasdaq \\
government & game & misidentified & zzz\_clinton & zzz\_democrat & onion & right & left & shares \\
product & played & referred & zzz\_sharon & sales & serve & win & win & rates \\
zzz\_medicare & night & zzz\_washington & minister & zzz\_social\_security & medium & percent & player & rate \\
study & set & son & fire & proposal & pound & ball & zzz\_giant & interest \\
\hline
\end{tabular}

\begin{tabular}{|c|c|c|c|c|c|c|c|c|}
\hline
zzz\_al\_gore & zzz\_george\_bush & car & book & zzz\_taliban & com & zzz\_bush & court & percent \\
campaign & president & race & children & attack & www & percent & case & number \\
president & zzz\_al\_gore & driver & ages & zzz\_afghanistan & site & campaign & law & group \\
zzz\_george\_bush & campaign & team & author & official & web & zzz\_enron & lawyer & rate \\
zzz\_bush & republican & won & read & military & sites & administration & federal & million \\
zzz\_clinton & zzz\_john\_mccain & win & newspaper & zzz\_u\_s & information & president & government & sales \\
vice & election & racing & web & zzz\_united\_states & online & zzz\_white\_house & decision & survey \\
presidential & zzz\_texas & track & writer & terrorist & mail & money & trial & according \\
million & presidential & season & written & war & internet & plan & zzz\_microsoft & study \\
democratic & political & lap & sales & bin & telegram & republican & right & quarter \\
night & zzz\_enron & point & find & laden & visit & company & judge & average \\
voter & governor & sport & history & zzz\_american & find & million & legal & economy \\
election & administration & seat & list & zzz\_bush & zzz\_internet & zzz\_republican & ruling & american \\
vote & democratic & races & word & government & computer & official & attorney & increase \\
plan & zzz\_white\_house & road & published & group & org & zzz\_texas & death & rose \\
zzz\_bill\_bradley & voter & run & school & forces & newspaper & election & system & black \\
ballot & nation & look & zzz\_new\_york & zzz\_pakistan & offer & show & company & student \\
zzz\_governor\_bush & public & right & right & country & free & political & zzz\_supreme\_court & level \\
republican & zzz\_clinton & zzz\_nascar & boy & leader & services & zzz\_mccain & election & school \\
zzz\_florida & zzz\_republican & drive & writing & american & company & energy & cases & season \\
right & candidate & zzz\_winston\_cup & american & afghan & official & zzz\_washington & prosecutor & poll \\
votes & point & owner & reading & troop & list & zzz\_united\_states & public & newspaper \\
poll & question & start & game & terrorism & user & voter & zzz\_florida & job \\
court & percent & big & reader & nation & companies & fund & ballot & consumer \\
candidates & zzz\_party & ago & won & zzz\_pentagon & customer & zzz\_al\_gore & states & government \\
\hline
\end{tabular}

\begin{tabular}{|c|c|c|c|c|c|c|c|c|}
\hline
company & show & game & computer & film & team & bill & cell & election \\
percent & network & games & system & movie & player & zzz\_senate & patient & ballot \\
million & season & season & program & director & season & law & human & vote \\
business & zzz\_nbc & play & zzz\_microsoft & play & game & right & research & voter \\
companies & zzz\_cb & goal & mail & character & coach & zzz\_white\_house & group & campaign \\
billion & program & king & software & actor & play & zzz\_congress & scientist & political \\
analyst & television & team & window & show & games & vote & zzz\_enron & votes \\
stock & series & won & web & movies & right & member & study & official \\
quarter & night & player & company & million & league & president & disease & zzz\_florida \\
executive & zzz\_new\_york & coach & million & part & million & legislation & information & democratic \\
deal & zzz\_abc & played & information & zzz\_hollywood & deal & zzz\_clinton & found & race \\
sales & tonight & period & need & look & manager & group & team & zzz\_republican \\
share & hour & left & technology & big & need & zzz\_house & public & recount \\
zzz\_enron & look & playing & user & young & contract & republican & doctor & republican \\
chief & zzz\_fox & night & security & music & guy & campaign & government & won \\
market & air & win & zzz\_internet & set & point & federal & death & leader \\
employees & viewer & right & problem & screen & played & money & cancer & candidate \\
customer & rating & com & internet & writer & baseball & election & researcher & zzz\_al\_gore \\
president & game & playoff & money & television & agent & support & stem & zzz\_party \\
product & early & power & home & making & fan & zzz\_republican & official & poll \\
executives & big & guy & network & love & playing & measure & problem & candidates \\
financial & talk & zzz\_new\_york & product & played & job & issue & called & party \\
earning & event & record & called & producer & free & passed & medical & presidential \\
operation & hit & shot & help & guy & sport & percent & director & win \\
cent & award & minutes & number & kind & basketball & billion & question & result \\
\hline
\end{tabular}

\begin{tabular}{|c|c|c|c|c|c|c|c|c|}
\hline
money & police & team & air & family & music & official & companies & president \\
million & officer & game & water & children & song & government & job & program \\
fund & official & win & million & home & group & zzz\_united\_states & worker & zzz\_bush \\
zzz\_enron & president & won & high & father & part & zzz\_china & company & group \\
campaign & government & zzz\_u\_s & building & mother & zzz\_new\_york & zzz\_u\_s & business & game \\
program & attack & play & power & son & company & zzz\_american & firm & member \\
group & case & games & plant & parent & million & country & zzz\_new\_york & zzz\_clinton \\
plan & told & official & plan & child & band & administration & attack & care \\
government & office & point & cost & friend & show & zzz\_clinton & president & leader \\
firm & member & run & hour & school & album & million & employees & health \\
company & public & home & system & boy & companies & nation & plan & zzz\_white\_house \\
pay & death & zzz\_united\_states & wind & wife & record & countries & need & vice \\
worker & group & sport & part & house & play & president & law & plan \\
help & zzz\_new\_york & zzz\_new\_york & weather & told & right & economic & percent & job \\
job & chief & attack & area & daughter & business & foreign & customer & children \\
political & black & tournament & home & kid & look & power & industry & patient \\
lawyer & lawyer & american & rain & night & artist & chinese & number & executive \\
member & prosecutor & percent & shower & help & home & zzz\_russia & cost & worker \\
account & security & minutes & front & care & industry & political & terrorist & doctor \\
effort & building & zzz\_olympic & program & left & member & plan & security & school \\
billion & campaign & final & billion & official & black & meeting & market & decision \\
employees & night & player & night & room & sound & leader & information & director \\
financial & hour & company & feet & money & night & trade & help & zzz\_congress \\
question & home & lead & low & hour & called & percent & official & administration \\
need & found & zzz\_washington & miles & job & fan & right & economy & chief \\
\hline
\end{tabular}

\begin{tabular}{|c|c|c|c|c|}
\hline
government & season & right & test & file \\
companies & team & zzz\_united\_states & zzz\_seattle\_post\_intelligencer & onlytest \\
political & won & american & zzz\_hearst\_news\_service & sport \\
country & race & war & zzz\_kansas\_city & notebook \\
president & win & student & look & zzz\_los\_angeles \\
campaign & attack & look & testing & onlyendpar \\
leader & home & need & houston & zzz\_joe\_haakenson\_san\_gabriel\_valley\_tribune \\
business & record & show & ellipses & zzz\_anaheim\_angel \\
election & games & home & anthrax & frontend \\
zzz\_bush & zzz\_u\_s & question & student & zzz\_seattle\_pi \\
win & final & black & glories & zzz\_seattle\_post\_intelligencer \\
war & zzz\_clinton & military & mark & zzz\_chuck \\
company & night & left & night & zzz\_abcdefg\_test \\
zzz\_internet & million & country & rare & added \\
billion & zzz\_olympic & com & zzz\_texas & zzz\_los\_angeles\_dodger \\
race & winning & women & result & read \\
power & coach & word & risk & zzz\_calif \\
support & championship & put & exam & output \\
market & patient & zzz\_american & system & email \\
team & playoff & help & scores & internet \\
democratic & victory & room & missile & zzz\_brian\_dohn \\
won & american & zzz\_u\_s & zzz\_washington & files \\
public & trial & zzz\_america & body & zzz\_scott\_wolf \\
web & medal & percent & according & wrote \\
industry & series & job & scientist & consumer \\
\hline
\end{tabular}

\end{center}

\end{landscape}

\end{document}